\documentclass[lettersize,journal]{IEEEtran}
\usepackage{amsmath, amsthm, amssymb, amsfonts}
\usepackage{dcolumn}
\usepackage{bm}
\usepackage{enumitem}
\usepackage{algorithmic}
\usepackage{algorithm}
\usepackage{array}
\usepackage[caption=false,font=normalsize,labelfont=sf,textfont=sf]{subfig}
\usepackage{textcomp}
\usepackage{stfloats}
\usepackage{url}
\usepackage{verbatim}
\usepackage{graphicx}
\usepackage{cite}
\usepackage{hyperref}
\usepackage{tikz}
\usepackage{balance}
\usetikzlibrary{positioning}

\newtheorem{proposition}{Proposition}
\newtheorem{definition}{Definition}

\begin{document}

\title{Efficient Generative Modeling with Unitary Matrix Product States Using Riemannian Optimization}

\author{Haotong Duan$^{*}$, Zhongming Chen$^{*}$, Ngai Wong$^\dagger$

\thanks{Corresponding author: Zhongming Chen.}
\thanks{$^*$Department of Mathematics, School of Sciences, Hangzhou Dianzi University, Hangzhou 310018, China. Email: \{htduan, zmchen\}@hdu.edu.cn.}
\thanks{$^\dagger$Department of Electrical and Electronic Engineering, The University of Hong Kong, Hong Kong. Email: nwong@eee.hku.hk.}
}



\maketitle
\IEEEpubidadjcol

\begin{abstract}
Tensor networks, which are originally developed for characterizing complex quantum many-body systems, have recently emerged as a powerful framework for capturing high-dimensional probability distributions with strong physical interpretability. This paper systematically studies matrix product states (MPS) for generative modeling and shows that unitary MPS, which is a tensor-network architecture that is both simple and expressive, offers clear benefits for unsupervised learning by reducing ambiguity in parameter updates and improving efficiency. To overcome the inefficiency of standard gradient-based MPS training, we develop a Riemannian optimization approach that casts probabilistic modeling as an optimization problem with manifold constraints, and further derive an efficient space-decoupling algorithm. Experiments on Bars-and-Stripes and EMNIST datasets demonstrate fast adaptation to data structure, stable updates, and strong performance while maintaining the efficiency and expressive power of MPS.
\end{abstract}


\section{\label{sec:level1}Introduction}
\IEEEPARstart{T}{ensor} networks were originally developed in condensed matter physics to accurately characterize the wave functions and entanglement structures of quantum many-body systems \cite{eisert2010colloquium,hastings2007area}. The seminal survey by Kolda and Bader \cite{kolda2009tensor} systematically introduced tensor decompositions into scientific computing and data analysis. The models such as the tensor train (TT) decomposition which is also known as the matrix product state (MPS) \cite{oseledets2011tensor} were brought into the machine learning community \cite{chen2018tensor, liu2023tensor}. In recent years, tensor networks have received growing attention at the interface of machine learning and quantum computing due to their strong expressive capacity in representing high-dimensional data and their efficiency in simulating quantum many-body systems \cite{orus2019tensor}. These advantages make tensor networks a promising paradigm for addressing the limitations of conventional generative models, which often struggle with high computational complexity and limited interpretability when processing high-dimensional or quantum-structured data.

Generative modeling, which is a typical kind of unsupervised learning that makes use of a huge amount of unlabeled data, lies at the heart of the rapid development of modern machine learning techniques.
Classical generative models such as Boltzmann machines \cite{ackley1985learning,cheng2018information} and generative adversarial networks (GANs) \cite{goodfellow2014generative} exhibit inherent scalability challenges when confronted with extremely high-dimensional data. This has motivated increasing interest in tensor network–based generative frameworks. A rigorous correspondence between probabilistic modeling and the MPS formalism was established in \cite{stokes2019probabilistic}, laying the theoretical foundation for using MPS as a probabilistic generative model. Early studies introduced MPS for dimensionality reduction and large-scale optimization \cite{cichocki2016tensor}. Subsequent empirical investigations further revealed that MPS achieves generative performance comparable to restricted Boltzmann machines (RBMs) on benchmark datasets such as MNIST and Bars-and-Stripes \cite{han2018unsupervised,li2018shortcut}. In addition, more recent theoretical work has demonstrated the strong expressive power of MPS when viewed through the lens of probabilistic modeling \cite{glasser2019expressive}.

Beyond the standard MPS structure, a variety of tensor network-based generative models have been explored. The tree tensor network (TTN) has been shown to effectively capture hierarchical data dependencies \cite{cheng2019tree,shi2006classical}. Two-dimensional tensor networks such as projected entangled pair states (PEPS) \cite{vieijra2022generative} and hierarchical tree architectures \cite{liu2019machine} naturally align with image modeling and completion tasks. The deep tree tensor network further integrates hierarchical structures with deep learning mechanisms, yielding improved representational capacity for computer vision tasks \cite{nie2025deep}. Meanwhile, novel tensor decompositions including tensor ring and tensor star models offer efficient tools for high-dimensional data compression and completion \cite{zhou2024tensor}, highlighting their potential as generative modeling frameworks.

Algorithmic advancements are essential for making tensor network-based generative models practically viable. Riemannian optimization has emerged as a powerful tool for solving constrained optimization problems on manifolds, providing significant improvements in training stability and efficiency. This framework generalizes classical optimization techniques from Euclidean spaces to Riemannian manifolds with intrinsic geometric structures \cite{uschmajew2020geometric,yang2025space}. Foundational works such as \cite{absil2008optimization} provide detailed theoretical treatments of matrix manifold algorithms and illustrate how geometric structures can be exploited to design stable and efficient optimization procedures. Many machine learning models naturally impose manifold constraints on their parameters such as the Stiefel manifold \cite{klemetsen2022neural}, the unit sphere \cite{peng2025normalized}, and low-rank tensor manifolds \cite{saragadam2024deeptensor,jacob2020structured}. In complex generative models, however, a single manifold constraint is often insufficient. Parameterization-based approaches can preserve manifold structure during optimization \cite{yang2025space}, while Riemannian optimization avoids the computational overhead and instability associated with Euclidean projection-based methods.


When training MPS-based generative models, local probabilities are determined by the relative values of normalization coefficients induced by the optimized MPS tensors. Under unconstrained optimization \cite{han2018unsupervised}, these coefficients might oscillate or adapt slowly because some gradient directions mainly rescale the entire MPS and thus do not affect local probabilities.
To address this, we propose a unitary MPS-based generative model that enforces a unit-sphere constraint, limiting optimization to directions that adjust the relative weights among MPS cores.
Combining a DMRG-inspired update scheme \cite{white1992density} with a space-decoupled strategy \cite{yang2025space} enables efficient updates at the intersections of tensor manifolds and supports independent, parallel optimization of MPS cores. Experiments show that, relative to alternating Euclidean gradient descent, the proposed Riemannian method maintains manifold constraints throughout training, yielding a more direct descent trajectory, fewer boundary oscillations, and markedly improved stability and efficiency.
The main contributions of this paper are:
\begin{itemize}
    \item A unitary MPS framework for generative modeling is proposed to remove global scaling degrees of freedom while enforcing tensor-norm or orthogonality constraints. 
    \item An efficient manifold-optimization approach is developed by combining DMRG-inspired updates with a space-decoupled strategy, applying Riemannian optimization at manifold intersections to enable parallel core updates.
    \item Strong generation performance is demonstrated, validating unitary MPS advantages in convergence stability, efficiency, and generation quality.
\end{itemize}

The remainder of this paper is organized as follows. Section~II introduces the necessary theoretical background, including the MPS representation and fundamental concepts in manifold optimization. Section~III presents the overall MPS-based generative modeling framework, formulates the associated optimization problem, and describes the proposed space-decoupled optimization algorithm. Section~IV first presents experimental results on the Bars-and-Stripes (BAS) dataset to evaluate the convergence speed and generative capability of the proposed method. It then reports results on the EMNIST dataset, where the proposed approach is compared with the baseline in \cite{han2018unsupervised}, with emphasis on convergence speed and generative performance. Section~V discusses future directions for tensor network–based generative modeling and analyzes the potential of Riemannian optimization in addressing related challenges.

\section{\label{sec:level2}Preliminaries}
This section provides a concise review of the theoretical background underpinning this work, including basic tensor operations, the matrix product state (MPS) representation, fundamental principles of Riemannian optimization, and the associated manifold structures. These concepts collectively establish the foundation for the model formulation and optimization strategies introduced in the subsequent sections.

\subsection{Tensor Preliminaries}
We first introduce some basic tensor notations used throughout this work.
A $d$th-order tensor is denoted by 
$\mathcal{A} \in \mathbb{R}^{n_1 \times n_2 \times \cdots \times n_d}$,
whose $(i_1, i_2, \ldots, i_d)$th element is written as 
$\mathcal{A}(i_1, i_2, \ldots, i_d)$,
where $i_k = 1, \ldots, n_k$ and $k = 1, \ldots, d$.
Here, $d$ is referred to as the order of the tensor, and $n_k$ as the dimension
of its $k$th mode. 
A first-order tensor corresponds to a vector, and a second-order tensor corresponds to a matrix. 

\begin{definition}[Frobenius norm]\label{def:frobenius}
For any two tensors 
$\mathcal{A}, \mathcal{B} \in \mathbb{R}^{n_1 \times \cdots \times n_d}$,
their inner product is defined as
\begin{equation}
\langle \mathcal{A}, \mathcal{B} \rangle 
= \sum_{i_1=1}^{n_1} \cdots \sum_{i_d=1}^{n_d} \mathcal{A}(i_1, \ldots, i_d) \, \mathcal{B}(i_1, \ldots, i_d).
\label{eq:inner}
\end{equation}
The corresponding Frobenius norm is given by
\begin{equation}
\| \mathcal{A} \|_F = 
\sqrt{ \langle \mathcal{A}, \mathcal{A} \rangle }.
\label{eq:frobenius}
\end{equation}
\end{definition}

\begin{definition}[Outer product]
For vectors $\mathbf{a}^{(k)} \in \mathbb{R}^{n_k}$, $k=1,\ldots,d$, their outer product is denoted by $\mathcal{C} = \mathbf{a}^{(1)} \circ \cdots \circ \mathbf{a}^{(d)} \in \mathbb{R}^{n_1 \times \cdots \times n_d}$ such that
$$\mathcal{C}(i_1, \ldots, i_d) = \mathbf{a}^{(1)}(i_1) \cdots \mathbf{a}^{(d)}(i_d).$$
\end{definition}

\begin{definition}[Tensor train decomposition]\label{def:TT}
Given a $d$th-order tensor $\mathcal{A} \in \mathbb{R}^{n_1 \times \cdots \times n_d}$,
the tensor train (TT) decomposition, also known as the matrix product state (MPS) representation,
factorizes $\mathcal{A}$ into a sequence of $d$ third-order tensors such that
\begin{equation*}
\mathcal{A}=\sum_{\alpha_{0}=1}^{r_0} \cdots \sum_{\alpha_{d}=1}^{r_d} A^{(1)}\left(\alpha_{0}, :, \alpha_{1}\right) \circ \cdots \circ A^{(d)}\left(\alpha_{d-1}, :, \alpha_{d}\right)
\label{eq:TT_general}
\end{equation*}
which can be expressed element-wise as
\begin{equation}
\mathcal{A}(v_{1},v_{2},\ldots,v_{d})=A^{(1)}(v_{1})A^{(2)}(v_{2})\cdots A^{(d)}(v_{d}) .  
\end{equation}
Here, $A^{(k)} \in \mathbb{R}^{r_{k-1} \times n_k \times r_k}$ ($k=1,2,\dots,d$) 
are the core tensors, with the boundary condition $r_0 = r_d = 1$ to ensure proper contraction at the edges. The $v_k$th slice of $A^{(k)}$ is denoted by $A^{(k)}(v_k) \in \mathbb{R}^{r_{k-1} \times r_k}$.  
The sequence $r = (r_1, r_2, \dots, r_{d-1})$ is referred to as the \emph{TT-rank}, also called the \emph{bond dimensions} in the MPS terminology. 
\end{definition}

For any MPS, the maximum bond dimension is denoted by  $r_{\rm max}$ and the average bond dimension is denoted by $r_{\rm mean}$. The MPS with unit Frobenius norm is called unitary MPS.

\begin{definition}[$k$-unfolding \cite{liu2022tensor}]
The $k$-unfolding operator directly reshapes the original tensor into a matrix with the first $k$ modes as the row while the rest modes as the column. For a $d$th-order tensor $\mathcal{A} \in \mathbb{R}^{n_{1} \times \cdots \times n_{d}}$, it can be expressed as $\mathbf{A}_{\langle k \rangle}$ whose entries meet
\[
\mathbf{A}_{\langle k \rangle}(j_{1}, j_{2}) = \mathcal{A}(i_{1}, \ldots, i_{k}, \ldots, i_{d}),
\]
\text{where }
\[
j_{1} = \overline{i_{1},\ldots,i_{k}}, 
\qquad 
j_{2} = \overline{i_{k+1},\ldots,i_{d}}
\]
denote the linear indices corresponding to the multi-indices according to the big-endian (colexicographic) notation, i.e., the mappings are explicitly given by
\[
j_{1}
= i_{k} 
+ (i_{k-1}-1) n_{k}
+ \cdots 
+ (i_{1}-1) n_{2} n_{3} \cdots n_{k},
\]
\[
j_{2}
= i_{d} 
+ (i_{d-1}-1) n_{d}
+ \cdots
+ (i_{k+1}-1) n_{k+2} n_{k+3} \cdots n_{d}.
\]
\end{definition}

\begin{proposition}[Theorem 2.1 of \cite{oseledets2011tensor}]\label{property:tt}
For any tensor $\mathcal{A} \in \mathbb{R}^{n_1 \times n_2 \times \cdots \times n_d}$, there always exist TT-cores  $A^{(k)} \in \mathbb{R}^{r_{k-1} \times n_k \times r_k}$ ($k=1,2,\dots,d$) satisfying (\ref{eq:TT_general}) with TT-ranks
$$ r_k = \mathrm{rank}(\mathbf{A}_{\langle k \rangle}), \quad k = 1,2, \ldots, d-1.$$
\end{proposition}

\begin{definition}[Left- and right-canonical]\label{def:lr-orthogonal}
For $k=1,\ldots, d$, the core tensor $A^{(k)} \in \mathbb{R}^{r_{k-1} \times n_k \times r_k}$ is called left-canonical if
\begin{equation}
\sum_{v_k=1}^{n_k} \big(A^{(k)}(v_k)\big)^\top A^{(k)}(v_k)
= I_{r_k},
\label{eq:left_canonical}
\end{equation}
and right-canonical if
\begin{equation}
\sum_{v_k=1}^{n_k} A^{(k)}(v_k)\big(A^{(k)}(v_k)\big)^\top
= I_{r_{k-1}} .
\label{eq:right_canonical}
\end{equation}
\end{definition}

\begin{proposition}[Mixed-canonical form of MPS \cite{oseledets2011tensor}]\label{property:MC-MPS}
For any MPS $\{A^{(1)},\dots,A^{(d)}\}$, one could select an index $k \in\{1,2,\ldots,d\}$ as the
\emph{orthogonality center} and transform the cores such that
\begin{align*}
&A^{(1)},\dots,A^{(k-1)} \text{ are left-canonical and} \nonumber\\
&A^{(k+1)},\dots,A^{(d)} \text{ are right-canonical.}
\label{eq:mixed_structure}
\end{align*}
Under this mixed-canonical structure, the contraction of the entire MPS satisfies
\begin{equation}
\|\mathcal{A}\|_F = \big\|A^{(k)}\big\|_F ,
\end{equation}
meaning that the Frobenius norm of the full tensor is completely concentrated at
the orthogonality center. 

Moreover, by contracting two adjacent cores $A^{(k)}$ and $A^{(k+1)}$ into a
fourth-order tensor 
\begin{equation}
\begin{split}
&A^{(k,k+1)}(\alpha_{k-1}, v_k, v_{k+1}, \alpha_{k+1}) = \\
&\quad \sum_{\alpha_k =1}^{r_k}A^{(k)}(\alpha_{k-1}, v_k, \alpha_{k})A^{(k+1)}(\alpha_k, v_{k+1}, \alpha_{k+1}) 
\end{split}
\label{eq:contracting}
\end{equation}
while keeping $A^{(1)},\dots,A^{(k-1)}$ left-canonical and
$A^{(k+2)},\dots,A^{(d)}$ right-canonical, the MPS admits a
two-site mixed-canonical form.
In this case, the Frobenius norm of the full tensor is concentrated on the merged
core, namely,
\begin{equation}
\|\mathcal{A}\|_F = \big\|A^{(k,k+1)}\big\|_F .
\end{equation}
\end{proposition}

\subsection{Preliminaries on Riemannian Optimization}
In this subsection, we present some basic concepts about Riemannian optimization which acts on a smooth manifold equipped with a Riemannian metric.

\begin{definition}[Tangent space]\label{def:tangent-space}
Let $\gamma : (-\epsilon, \epsilon) \rightarrow \mathcal{M}$ be a smooth curve 
on the manifold $\mathcal{M}$ with $\gamma(0) = x$. 
The tangent vector of the curve $\gamma$ at the point $x$ is the operator
$\dot{\gamma}(0)$, which acts on any smooth function 
$f \in C^\infty(\mathcal{M})$ as
\begin{equation}
\dot{\gamma}(0) f := \left. \frac{d}{dt} f(\gamma(t)) \right|_{t=0}.
\end{equation}

The collection of all tangent vectors at $x$ forms a vector space, called
the tangent space at $x$, denoted by $T_x \mathcal{M}$.    
\end{definition}

\begin{definition}[Riemannian gradient]\label{def:riemann_gradient}
Let $(\mathcal{M}, g)$ be a Riemannian manifold, where $g$ is a Riemannian metric 
that assigns a positive-definite inner product on the tangent space $T_x \mathcal{M}$ at 
each point $x \in \mathcal{M}$:
\[
g_x : T_x \mathcal{M} \times T_x \mathcal{M} \rightarrow \mathbb{R}.
\]
Let $f: \mathcal{M} \rightarrow \mathbb{R}$ be a smooth function. 
The Riemannian gradient of $f$ at $x \in \mathcal{M}$, denoted by 
$\mathrm{grad} f(x)$, is the unique vector in $T_x \mathcal{M}$ satisfying
\begin{equation}
g_x \big( \mathrm{grad} f(x), v \big) = Df(x)[v], 
\quad \forall v \in T_x \mathcal{M},
\label{eq:riemannian_gradient}
\end{equation}
where $Df(x)[v]$ denotes the directional derivative of $f$ at $x$ along the vector $v$.
\end{definition}

\begin{proposition}[Computation of Riemannian gradient \cite{absil2008optimization}]
Let $\mathcal{M}$ be the Riemannian manifold embedded in Euclidean space $\mathbb{R}^n$.
The Riemannian gradient of $f$ at $x \in \mathcal{M}$ can be computed by projecting the Euclidean gradient 
${\nabla} f(x)$ onto the tangent space $T_x \mathcal{M}$:
\begin{equation}
\mathrm{grad} f(x) = P_{T_x \mathcal{M}} \big( {\nabla} f(x) \big),
\end{equation}
where $P_{T_x \mathcal{M}}$ is the orthogonal projection operator onto $T_x \mathcal{M}$ under the given Riemannian metric. 
\end{proposition}


\begin{proposition}[Transversality \cite{LM08}]
Let $\mathcal{M}_1$ and $\mathcal{M}_2$ be smooth manifolds embedded in 
$\mathbb{R}^n$ and $\mathcal{M}_1 \cap \mathcal{M}_2 \neq \emptyset $. It follows that  $\mathcal{M}_1 \cap \mathcal{M}_2$ is still a smooth manifold if and only if $\mathcal{M}_1$ and $\mathcal{M}_2$ intersect 
transversely, i.e.,
\begin{equation}
T_x \mathcal{M}_1 + T_x \mathcal{M}_2 = \mathbb{R}^{n}
\label{eq:transverse}
\end{equation}
for any $x \in \mathcal{M}_1 \cap \mathcal{M}_2$.    
\end{proposition}

\begin{definition}[Retraction]\label{def:retraction}
A retraction on a manifold $\mathcal{M}$ is a smooth mapping $R$ from the tangent bundle $T \mathcal{M}$ onto $\mathcal{M}$ with the following properties. Let $R_x$ denote the restriction of $R$ to $T_x \mathcal{M}$.  
\begin{itemize}
    \item[(i)] $R_x(0_x) = x$, where $0_x$ denotes the zero element of $T_x \mathcal{M}$. 
    \item[(ii)] $R_x$ satisfies  $DR_x(0_x) = id_{T_x \mathcal{M}}$, where $id_{T_x \mathcal{M}}$ denotes the identity mapping on $T_x \mathcal{M}$.
\end{itemize}
\end{definition}

In fact, the exponential mapping is the ideal way to map a tangent vector to a point on the manifold by traveling along the geodesic with initial velocity given by the tangent vector \cite{absil2008optimization}.
It’s geometrically canonical, but often too expensive or inconvenient to compute. In practice, a retraction is used instead because it gives a cheap, well-behaved approximation of the exponential mapping that is sufficient for convergence guarantees.
In Riemannian gradient-based optimization, the update rule at iteration $k$ 
is given by
\begin{equation}
x_{k+1} = R_{x_k} \Big( - \theta \, \mathrm{grad} f(x_k) \Big),
\label{eq:retraction_update}
\end{equation}
where $R_x(\xi)$ is a retraction mapping that maps a tangent vector 
$\xi \in T_x \mathcal{M}$ back to a point on the manifold $\mathcal{M}$, 
and $\theta > 0$ is the learning rate.

\subsection{Relevant Manifolds}
In this subsection, we introduce two  smooth manifolds embedded in 
$\mathbb{R}^{m \times n}$: unit sphere manifold $\mathcal{S}_{m\times n}$ and fixed-rank manifold $\mathcal{M}_k$. Their definitions are given below, along with some basic properties used in the subsequent analysis.

\subsubsection{Unit Sphere Manifold}
Define the manifold
\begin{equation}
\mathcal{S}_{m\times n} = \{ X \in \mathbb{R}^{m \times n} \mid \| X \|_F = 1 \},
\label{unit_sphere_manifold}
\end{equation}
which is the unit sphere endowed with the Frobenius norm. 
It is a smooth manifold of dimension $mn - 1$.

For any $X \in \mathcal{S}_{m\times n}$, the tangent space $T_X \mathcal{S}_{m\times n}$ consists 
of all matrices that are orthogonal to $X$ with respect to the Frobenius inner product:
\begin{equation}
\begin{split}
T_X \mathcal{S}_{m\times n} = \{ \dot{X} \in \mathbb{R}^{m \times n} \mid 
\langle X, \dot{X} \rangle = \mathrm{tr}(X^\top \dot{X}) = 0 \}.
\end{split}
\label{eq:unit_sphere_tangent}
\end{equation}


\subsubsection{Fixed-Rank Manifold}
Define the manifold
\begin{equation}
\mathcal{M}_k = \{ X \in \mathbb{R}^{m \times n} \mid \mathrm{rank}(X) = k \},
\label{fix_rank_manifold}
\end{equation}
which is the fixed-rank-$k$ manifold. It is a smooth manifold of dimension 
$k(m+n-k)$.

For any $X \in \mathcal{M}_k$, let $X$ have a rank decomposition 
$X = U V^\top$ with $U \in \mathbb{R}^{m \times k}$ and $V \in \mathbb{R}^{n \times k}$. 
Then the tangent space at $X$ is given by
\begin{equation}
T_X \mathcal{M}_k = \{ \dot{U} V^\top + U \dot{V}^\top \mid 
\dot{U} \in \mathbb{R}^{m \times k}, \ \dot{V} \in \mathbb{R}^{n \times k} \}.
\label{eq:fixed_rank_tangent}
\end{equation}
It is worth noting that when $k \leq r$, the fixed-rank manifold $\mathcal{M}_k$ is an embedded smooth submanifold of the
low-rank set
\begin{equation}
\mathcal{M}_{\le r} = \{ X \in \mathbb{R}^{m \times n} \mid \mathrm{rank}(X) \le r \},
\label{eq:low_rank}
\end{equation}
which, in contrast, is not a smooth manifold everywhere due to the presence of singular points at matrices 
with rank strictly smaller than $r$. The manifold $\mathcal{M}_k$ can therefore be viewed as the smooth 
stratum of $\mathcal{M}_{\le r}$ corresponding to matrices of maximal rank $r$ within this set.

In Riemannian optimization, intersections of multiple manifolds frequently arise. In practice, the tangent space of such an intersection can be approximated via iterative projections, after which standard Riemannian optimization procedures can be applied to search for an optimal solution. However, in tensor optimization problems, the optimal tensor rank is often unknown. From a theoretical perspective, identifying the global optimum would require an exhaustive search over all possible ranks, which is computationally infeasible. Furthermore, the set of low-rank tensors does not constitute a smooth manifold, and consequently, its intersection with other manifolds is not guaranteed to possess smoothness properties.

\section{A space-decoupling Method for optimization on unitary MPS}
In this section, 
we first review the standard MPS generative model and then introduce the proposed generative model based on unitary MPS (UMPS), which incorporates a normalization constraint to enhance training stability and probabilistic interpretability. The UMPS model naturally leads to optimization problems defined over structured tensor manifolds, where the objective is to maximize the likelihood or other relevant criteria. For the UMPS model, we formulate the corresponding optimization problem and employ a space-decoupling strategy, which enables individual core tensors to be optimized independently while adhering to manifold constraints. This approach improves computational efficiency and supports parallel updates, making it particularly suitable for high-dimensional generative tasks.

\subsection{MPS for generative model}
With the objective of unsupervised learning, Han \textit{et al.}~\cite{han2018unsupervised} proposed a generative model based on the Matrix Product State representation (referred to as the MPS model), which is inspired by the probabilistic interpretation of wavefunctions in quantum physics.
Consider a dataset $\mathcal{T}$ consisting of binary strings 
$\mathbf{v}=(v_1,v_2,\cdots,v_d) \in \mathcal{V} = \{0,1\}^{\otimes d}$, 
which may contain duplicates and can be mapped to the basis vectors of a Hilbert space with dimension $2^d$. Later, we will regard $\mathcal{T}$ as the training set.
Let $|\mathcal{T}|$ represent the number of elements in set $\mathcal{T}$. Formally, the empirical distribution $\hat{P}$ obtained during training is given by:
\begin{equation}
\hat{P}(\mathbf{v}) = \frac{\text{frequency of } \mathbf{v} }{|\mathcal{T}|}, \quad \text{for } \mathbf{v} \in \mathcal{T}.    
\end{equation}
The MPS parameterizes the wavefunction as follows:
\begin{equation*}
\begin{split}
&\Psi(\mathbf{v};A^{(1)}, \cdots, A^{(d)}) 
= \mathrm{Tr}\!\left(
A^{(1)}(v_1)
\cdots
A^{(d)}(v_d)
\right) \\
&= \sum_{\alpha_{0}=1}^{r_0} \cdots \sum_{\alpha_{d}=1}^{r_d} 
A^{(1)}\left(\alpha_{0}, v_1, \alpha_{1}\right)  
\cdots  
A^{(d)}\left(\alpha_{d-1}, v_d, \alpha_{d}\right),
\end{split}
\label{eq:mps_wavefunction}
\end{equation*}
where $A^{(k)} \in \mathbb{R}^{r_{k-1} \times 2 \times r_{k}}$, $k=1,2,\ldots, d$.
When the optimization of parameters in the MPS is not considered, we simply use 
$\Psi(\mathbf{v})$ to represent $\Psi(\mathbf{v};A^{(1)}, \cdots, A^{(d)})$.

The model represents the probability distribution of data as the squared norm of the MPS wavefunction:
\begin{equation}
P(\mathbf{v}) = \frac{|\Psi(\mathbf{v})|^2}{Z},
\label{eq:mps_prob}
\end{equation}
where $Z = \sum_{\mathbf{v} \in \mathcal{V}} |\Psi(\mathbf{v})|^2$
denotes the normalization factor ensuring that the total probability sums to one.
Once the MPS form of the wavefunction is established, the model learns the probability distribution by optimizing the parameters of $\Psi(\mathbf{v};A^{(1)}, \cdots, A^{(d)})$ 
so that the resulting distribution approximates the empirical data distribution $\hat{P}(\mathbf{v})$.
A common approach is to minimize the \textbf{Negative Log-Likelihood} (NLL), 
which is equivalent to minimizing the Kullback–Leibler divergence between the empirical distribution $\hat{P}(\mathbf{v})$ and the model distributions $P(\mathbf{v})$.
Thus, training an MPS generative model based on the NLL criterion is equivalent to solving the following optimization problem:
\begin{equation*}
\begin{aligned}
\min_{A^{(1)}, \cdots, A^{(d)}}
& \mathcal{L}
= -\frac{1}{|\mathcal{T}|}
\sum_{\mathbf{v} \in \mathcal{T}}
\ln \frac{|\Psi(\mathbf{v};A^{(1)}, \cdots, A^{(d)})|^2}{\sum_{\mathbf{v} \in \mathcal{V}} |\Psi(\mathbf{v};A^{(1)}, \cdots, A^{(d)})|^2}.
\end{aligned}
\label{eq:mps_loss}
\end{equation*}

In two-site DMRG \cite{stoudenmire2012studying, white1992density}, the MPS parameters are updated by optimizing two neighboring tensors simultaneously, followed by a truncation based on the reduced density matrix to retain the most relevant states. Specifically, when updating the $k$th and $(k+1)$th core tensors, we fix the remaining core tensors. By contracting two adjacent tensors $A^{(k)}$ and $A^{(k+1)}$, we obtain a fourth-order tensor denoted by $A^{(k,k+1)} \in \mathbb{R}^{r_{k-1} \times 2 \times 2 \times r_{k+1}}$ satisfying Eq. (\ref{eq:contracting}).
It follows that
\begin{equation*}
\begin{split}
&\Psi(\mathbf{v};A^{(1)}, \cdots, A^{(d)})\\[6pt]
&=
\sum_{\alpha_{i}(i \neq k)}
A^{(1)}(\alpha_{0}, v_1, \alpha_{1}) \cdots 
A^{(k,k+1)}(\alpha_{k-1}, v_k, v_{k+1}, \alpha_{k+1}) \\[2pt]
&\quad\ 
\cdots\,
A^{(d)}(\alpha_{d-1}, v_d, \alpha_{d}) \\[6pt]
&=
\Psi(\mathbf{v};A^{(1)}, \cdots, A^{(k,k+1)}, \cdots, A^{(d)})
:= \Psi(\mathbf{v};A^{(k,k+1)}).
\end{split}
\end{equation*}

Then, we compute the derivative of the negative log-likelihood with respect to $A^{(k,k+1)}$:
\begin{equation}
\frac{\partial \mathcal{L}}
{\partial A^{(k,k+1)}}
= \frac{Z'}{Z}
- \frac{2}{|\mathcal{T}|}
\sum_{\mathbf{v} \in \mathcal{T}}
\frac{\Psi'(\mathbf{v};A^{(k,k+1)})}{\Psi(\mathbf{v};A^{(k,k+1)})},
\label{eq:mps_gradient}
\end{equation}
where $\Psi'(\mathbf{v};A^{(k,k+1)})$ denotes the derivative of the MPS with  respect to the tensor element of $A^{(k,k+1)}$ and $Z' = 2 \sum_{\mathbf{v} \in \mathcal{V}} \Psi'(\mathbf{v};A^{(k,k+1)})\Psi(\mathbf{v};A^{(k,k+1)})$. 

Since the optimization problem arising from the above model is unconstrained, a natural approach is to employ gradient-based methods to search for an optimal solution. However, for the model distribution, uniformly scaling all values $\Psi({\mathbf{v}})$ associated with each element $\mathbf{v}$ by a common factor yields an equivalent probability distribution, as only the overall normalization is affected. This scale non-uniqueness can impede gradient-based optimization, causing iterates to oscillate among multiple equivalent optima or converge slowly due to the presence of flat directions in the objective landscape.

A natural way to mitigate this issue is to enforce the constraint \(Z = 1\) throughout the update process. This approach, often referred to as a gradient projection method, ensures that each iterate remains properly normalized and prevents the optimization from drifting along the redundant scaling direction. Simultaneously, the gradient expression further simplifies to:
\begin{equation}
\frac{\partial \mathcal{L}}
{\partial A^{(k,k+1)}}
= Z' 
- \frac{2}{|\mathcal{T}|}
\sum_{\mathbf{v} \in \mathcal{T}}
\frac{\Psi'(\mathbf{v};A^{(k,k+1)})}{\Psi(\mathbf{v};A^{(k,k+1)})}.
\label{eq:mps_gradient_simplified}
\end{equation}

Although the above strategy alleviates the issue to some extent, the projection step 
inevitably introduces a loss in the effective update progress, thereby reducing the 
overall efficiency of the optimization.
To eliminate this ambiguity, we augment the optimization problem by using unitary MPS. 
Fixing the partition function in this way removes the redundant degree of freedom 
and forces the optimization to converge toward a specific representative of the 
equivalence class of optimal solutions. As a result, the optimization becomes more 
stable and exhibits improved convergence speed and accuracy.

\subsection{Unitary MPS for generative model}
In this part, we introduce the proposed  generative model based on unitary Matrix Product State (referred to as the UMPS model), which is derived from the conventional MPS model by incorporating a normalization constraint $Z=1$. The motivation behind this modification is to eliminate redundant equivalence classes that naturally arise during the parameter update process. By constraining the MPS cores to a normalized form, the UMPS model effectively reduces the search space and accelerates convergence to the optimal solution. 

Specifically, we impose a normalization constraint $Z = 1$ 
to maintain the consistency of the probability distribution $P(\mathbf{v})$ during optimization, 
resulting in the following constrained optimization problem:
\begin{equation}
\begin{aligned}
\min_{A^{(1)}, \cdots, A^{(d)}} \ \ & 
\mathcal{L}
= -\frac{1}{|\mathcal{T}|}
\sum_{\mathbf{v} \in \mathcal{T}}
\ln |\Psi(\mathbf{v};A^{(1)}, \cdots, A^{(d)})|^2 \\
\text{subject to} \quad &
\sum_{\mathbf{v} \in \mathcal{V}} |\Psi(\mathbf{v};A^{(1)}, \cdots, A^{(d)})|^2 = 1.
\end{aligned}
\tag{P1}
\label{eq:nmps_constraint}
\end{equation}
Inspired by the update strategy of MPS model, (\ref{eq:nmps_constraint}) is optimized using a two-site alternating scheme. We define one $loop$ as sweeping from $k=1$ to $d-1$ and then returning to $k=1$. 
To be specific, each neighboring tensor $A^{(k,k+1)}$ is updated in the order
\[ A^{(1,2)}\rightarrow A^{(2,3)} \rightarrow \cdots \rightarrow A^{(d-1,d)}  \rightarrow \cdots \rightarrow A^{(1,2)} \rightarrow \cdots \]
until convergence.
According to Proposition~\ref{property:MC-MPS}, if the MPS is in the mixed-canonical form, we have $Z=\|A^{(k,k+1)}\|_F^2$ and $Z' = 2 A^{(k,k+1)}$ which simplify the computation of $Z$ and the Euclidean gradient given by (\ref{eq:mps_gradient_simplified}). 
Therefore, maintaining the mixed-canonical form of the MPS throughout the optimization is essential for computational simplicity. 

Specifically, we initialize the MPS such that $A^{(2)}, \ldots, A^{(d)}$ are right-canonical and $\|A^{(1)}\|_F = 1$. 
At this time, $Z = \|A^{(1,2)}\|_F^{2} = \|A^{(1)}\|_F^{2} = 1$.
Without loss of generality, after updating $A^{(k,k+1)}$, we perform the singular value decomposition (SVD) on the updated $A_{\langle 2 \rangle}^{(k,k+1)}$, i.e., $\mathrm{SVD}(A_{\langle 2 \rangle}^{(k,k+1)}) = U_k \Sigma_k V_k^\top$. If the next step updates 
$A^{(k+1,k+2)}$, we set $A^{(k)}_{\langle 2 \rangle} = U_k$ and $A^{(k+1)}_{\langle 1 \rangle} = \Sigma_k V_k^\top$. This makes $A^{(1)}, \ldots, A^{(k)}$ left-canonical and $A^{(k+2)}, \ldots, A^{(d)}$ right-canonical. If the next step updates $A^{(k-1,k)}$, we set $A^{(k)}_{\langle 2 \rangle} = U_k \Sigma_k$ and $A^{(k+1)}_{\langle 1 \rangle} = V_k^\top$ to maintain the mixed-canonical condition. This procedure guaranties that the condition $Z = \|A_{\langle 2 \rangle}^{(k,k+1)}\|_F^2$ is maintained throughout the entire update process. In order to meet the constraints of (\ref{eq:nmps_constraint}), the condition $\|A_{\langle 2 \rangle}^{(k,k+1)}\|_F^2 = 1$ is enforced at this step of the algorithm.
If we continue to employ the previously used gradient descent method to solve the problem, a truncated $\mathrm{SVD}$ is applied to 
\( A_{\langle 2 \rangle}^{(k,k+1)} \) in order to ensure that the maximum bond dimension $r_{max}$
does not exceed the prescribed threshold.
However, the truncation process breaks the manifold constraint in ~(\ref{eq:nmps_constraint}), 
which requires us to simultaneously preserve the unit-norm property 
of $A_{\langle 2 \rangle}^{(k,k+1)}$ 
and enforce its low-rank structure in the subproblem.

In practice, to simplify the model, we may constrain the tensor to lie 
on a fixed-rank manifold during optimization. 
However, this introduces substantial redundancy, 
and in many cases, only an upper bound of the bond dimension $r_{\rm max}$ is known while the optimal fixed rank remains indeterminate. 
That is, we need to search for a better tensor 
$A_{\langle 2 \rangle}^{(k,k+1)}$ within the intersection of the 
low-rank set $\mathcal{M}_{\le r_{\rm max}}$ 
and the unit sphere manifold $\mathcal{S}_{m \times n}$. 
This leads to the following subproblem of (\ref{eq:nmps_constraint}):
\begin{equation}
\begin{aligned}
\min_{A_{\langle 2 \rangle}^{(k,k+1)}} \quad
& \mathcal{L} 
= -\frac{1}{|\mathcal{T}|} 
\sum_{\mathbf{v} \in \mathcal{T}} 
\ln | \Psi\big(\mathbf{v}; A_{\langle 2 \rangle}^{(k,k+1)}\big) |^2 \\
\text{subject to} \quad
& \|A_{\langle 2 \rangle}^{(k,k+1)}\|_F^2 = 1, \\
& \mathrm{rank}\!\big(A_{\langle 2 \rangle}^{(k,k+1)}\big) \le r_{\rm max}.
\end{aligned}
\tag{P2}
\label{eq:nmps_subproblem_compact}
\end{equation}

Evidently, the above formulation is essentially an optimization problem defined over the intersection of the unit sphere and a low-rank matrix set. 
Due to the constraints, the problem becomes not only nonlinear but also strongly nonconvex, making direct optimization extremely challenging.
A straightforward approach is to enumerate all possible rank values and determine the optimal rank selection. According to Appendix \ref{transverse}, the unit sphere manifold and the fixed-rank manifold intersect transversely. Therefore, once the optimal rank is identified, the optimal solution can be obtained by projection onto the unit-sphere manifold. However, this enumeration strategy incurs exponential computational cost, rendering it impractical for applications.

In fact, constrained problems on non-smooth manifolds are not without direct solution strategies. In the next subsection, we design a spatially decoupled optimization method tailored for (\ref{eq:nmps_subproblem_compact}), which separates the intertwined low-rank constraint and unit-sphere constraint into two independent spaces. This decoupling allows the optimization to be carried out on smooth manifolds, where efficient and well-defined Riemannian optimization procedures become feasible.

\subsection{Space-decoupling method for UMPS}
The feasible set of (\ref{eq:nmps_subproblem_compact}) is given by the intersection of a low-rank matrix set and a smooth manifold, which results in a domain that is highly nonconvex and nonsmooth. Consequently, traditional optimization methods encounter substantial difficulties, including the absence of tractable projection operators, unclear optimality conditions, and potential instability caused by uncontrolled oscillations along scale-dependent directions. 

To systematically address these challenges, we adopt the \emph{space-decoupling} framework proposed by Yang \emph{et al.} \cite{yang2025space}, which provides an efficient computational paradigm tailored for low-rank optimization problems endowed with orthogonally invariant constraints. 
The central idea of this framework is to decouple variables and localize constraints so that the original manifold-constrained optimization problem can be transformed into a sequence of simpler and lower-dimensional subproblems, thereby significantly reducing computational complexity.

Before applying the framework to (\ref{eq:nmps_subproblem_compact}), we first consider the following general optimization model with a unit Frobenius-norm constraint and a low-rank constraint:
\begin{equation}
\begin{aligned}
\min_{X \in \mathbb{R}^{m \times n}} \quad &
 f(X), \\
\text{subject to} \quad &
\mathrm{rank}(X) \le r, \\
&\|X\|_F^2 = 1. 
\end{aligned}
\label{eq:general_lowrank_unitF}
\tag{P3}
\end{equation}
It can be seen that this is an optimization problem constrained on $\mathcal{M}_{\leq r} \cap \mathcal{S}_{m\times n}$ which is a nonsmooth manifold.

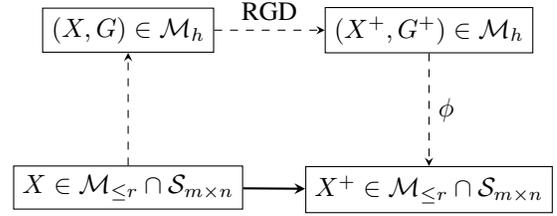
\begin{figure}[t]
\centering
\begin{tikzpicture}[>=stealth, node distance=4cm]
\node (a) [draw, rectangle] {$(X,G) \in \mathcal{M}_{h}$};
\node (b) [draw, rectangle, right of=a] {$(X^{+},G^{+}) \in \mathcal{M}_{h}$};
\node (c) [draw, rectangle, below=1.5cm of a] {$X \in \mathcal{M}_{\leq r} \cap \mathcal{S}_{m\times n}$};
\node (d) [draw, rectangle, below=1.5cm of b] {$X^{+} \in \mathcal{M}_{\leq r} \cap \mathcal{S}_{m\times n}$};
\draw[dashed, ->] (a) -- node[above, midway] {RGD} (b);
\draw[dashed, ->] (c) -- (a);
\draw[dashed, ->] (b) -- node[right, midway] {$\phi$} (d);
\draw[->, thick] (c) -- (d);
\end{tikzpicture}
\caption{Optimization flowchart of the space-decoupling method, where dashed lines indicate the optimization process.}
\label{Space Decoupling Method}
\end{figure}

To apply the space–decoupling framework to the above optimization model, the key idea is to parameterize the feasible set so that all constraints are automatically preserved throughout the optimization procedure. The flowchart is shown in Fig.~\ref{Space Decoupling Method}. Specifically, we decouple the intertwined constraints into two separate spaces, thereby enabling the application of Riemannian gradient descent (RGD) for optimization on the smooth manifold. For this purpose, we introduce the following space-decoupling parameterization for the constraints $\mathcal{M}_{\leq r} \cap \mathcal{S}_{m\times n}$:
\begin{equation*}
\mathcal{M}_{h}
=
\left\{
(X,G)\in \mathbb{R}^{m\times n} \times \mathrm{Gr}(n,n-r)
\;\big|\;
\begin{aligned}
& XG = 0, \\
& \|X\|_F = 1
\end{aligned}
\right\},
\label{eq:Mh_def}
\end{equation*}
where 
\[
\mathrm{Gr}(n,n-r)
=
\left\{
G \in \mathrm{Sym}(n)
:\;
G^2 = G,\;
\mathrm{rank}(G) = n-r
\right\}
\]
and $\mathrm{Sym}(n)$ denotes the space of real symmetric $n\times n$ matrices.
As shown in \cite{yang2025space}, $\mathcal{M}_{h}$ is a smooth manifold.

With the manifold $\mathcal{M}_h$ defined above, we circumvent the nonsmooth constraints inherent in the original problem (\ref{eq:general_lowrank_unitF}) 
by exploiting the smooth parameterization $\mathcal{M}_h$ with the mapping $\phi: \mathbb{R}^{m \times n}\times \mathrm{Sym}(n) \rightarrow \mathbb{R}^{m \times n}: (X,G) \mapsto X$ satisfying $\phi(\mathcal{M}_h) = \mathcal{M}_{\leq r} \cap \mathcal{S}_{m\times n}$. 
This enables us to lift the problem to a smooth Riemannian optimization problem on the abstract manifold $\mathcal{M}_h$:
\begin{equation}
\min_{(X,G) \in \mathcal{M}_h} \bar{f}(X,G) 
:= f\bigl(\phi(X,G)\bigr).
\tag{$\text{P-}\mathcal{M}_h$}
\label{eq:equal_problem_p3}
\end{equation}
Crucially, (\ref{eq:general_lowrank_unitF}) and (\ref{eq:equal_problem_p3}) share the same optimal value, 
and their optimal solutions are in one-to-one correspondence via the parameterization $\phi$. 
This smooth reformulation allows us to directly apply the rich toolbox of Riemannian optimization algorithms 
and their associated theoretical guarantees to solve the originally nonsmooth problem.

In what follows, we describe the tangent space $T_{(X,G)}\mathcal{M}_{h}$ and define the Riemannian metric on $\mathcal{M}_h$ which are needed for the Riemannain gradient descent update. Note that for any $(X,G)\in\mathcal{M}_{h}$, there always exists a \emph{representation} $(H, V)\in \mathbb{R}^{m \times r} \times \mathrm{St}(n,r)$ such that
\begin{equation}
X = H V^\top,
\qquad
G = I - V V^\top,
\label{eq:XG_factorization}
\end{equation}
where $\mathrm{St}(n,r)$ denotes the set of $n\times r$ matrices with orthonormal columns and $H \in \mathbb{R}^{m\times r}$ satisfies $\|H\|_F = 1$.
Under the space-decoupling framework, the tangent space of \((X,G)\in \mathcal{M}_{h}\) is given by
\begin{equation}
\label{eq:Mh_tangent}
\begin{aligned}
T_{(X,G)}\mathcal{M}_{h} & 
= \Big\{ ( K V^{\top} + H V_{p}^{\top},  - V_{p} V^{\top} - V V_{p}^{\top}) \mid \\
 & \ K \in T_{H}\mathcal{S}_{m \times r},\  V_p \in \mathbb{R}^{n \times r}, \ V^{\top} V_{p} = 0 \,\Big\}.
\end{aligned}
\end{equation}
Here, \(\mathcal{S}_{m \times r}=\{H \in \mathbb{R}^{m\times r} \mid \|H\|_{F}=1\}\) denotes the unit sphere manifold embedded in $\mathbb{R}^{m \times r}$ and the tangent space $T_{H}\mathcal{S}_{m \times r}$ is given by Eq. \eqref{eq:unit_sphere_tangent}. For any $(\eta,\zeta) \in T_{(X,G)}\mathcal{M}_{h}$, the corresponding $(K, V_p) \in \mathbb{R}^{m \times r} \times \mathbb{R}^{n \times r}$ in Eq. \eqref{eq:Mh_tangent} is also called a \emph{representation} of $(\eta,\zeta)$.

The manifold $\mathcal{M}_{h}$ is treated as an embedded submanifold of the Euclidean ambient space 
$\mathbb{R}^{m \times n} \times \mathrm{Sym}(n)$ equipped with the weighted inner product
\[
\langle (E_1,Z_1),(E_2,Z_2) \rangle_\omega 
= \langle E_1,E_2 \rangle + \omega \langle Z_1,Z_2 \rangle,
\]
where \(\omega > 0\) is the weighting
parameter.
The induced Riemannian metric on $\mathcal{M}_h$ is obtained by restricting this inner product to the tangent space $T_{(X,G)}\mathcal{M}_h$ at each point $(X,G) \in \mathcal{M}_h$. Taking into account the representations---$(H,V), (K_1, V_{p,1})$ and $(K_2, V_{p,2})$---for $(X,G)$ and $(\eta_1,\zeta_1)$, $(\eta_2,\zeta_2) \in T_{(X,G)}\mathcal{M}_{h}$, the metric takes the form
\[
\langle (\eta_1,\zeta_1), (\eta_2,\zeta_2) \rangle 
= \langle K_1,K_2 \rangle + \langle V_{p,1}, V_{p,2}{M_{H,\omega}} \rangle
\]
with the positive-definite weighting matrix:
\begin{equation}
M_{H,\omega} := 2\omega I_r + H^\top H.
\label{eq:weight_matrix}
\end{equation}
Here $I_r$ denotes the identity matrix of size $r\times r$ and we set $\omega=1$ throughout this work. 
This weighted metric provides the geometric foundation for Riemannian optimization algorithms.

Moreover, the Riemannian gradient of the lifted objective 
\(\bar{f}(X,G)=f(X)\) on \(\mathcal{M}_{h}\),
denoted by $\mathrm{grad} \bar{f}(X,G)$, 
is the unique vector in $T_{(X,G)}\mathcal{M}_{h}$
which can be represented by
\begin{equation}
\label{eq:Mh_gradient}
\begin{split}
\bar{K} &= \nabla f(X) V - \langle \nabla f(X) V,\  H \rangle \, H, \\
\bar{V}_{p} &= G \nabla f(X)^{\top} H \, M_{H,\omega}^{-1}  \  ,
\end{split}
\end{equation}
where $(H, V)$ is a representation of $(X, G) \in \mathcal{M}_{h}$, $\nabla f(X)$ denotes the Euclidean gradient of $f$ with respect to $X$ which could be computed via (\ref{eq:mps_gradient_simplified}), and $M_{H,\omega}$ is defined in Eq. \eqref{eq:weight_matrix}.

Within the Riemannian optimization framework, the gradient representation alone is insufficient. 
To preserve the manifold's geometric structure, a \emph{retraction} (see Definition~\ref{def:retraction}) is introduced, which projects the updated point back onto the manifold. 
For any $(X,G)\in \mathcal{M}_h$ and $(\eta, \zeta) \in T_{(X,G)} \mathcal{M}_h$, let $(H,V)$ and $(K,V_p)$ be the representations, respectively.
The first-order retraction for (\ref{eq:equal_problem_p3}) can be given by
\begin{equation}
\label{eq:retraction_nmps}
R_{(X,G)}(\eta, \zeta) =
\Big(
\frac{H+ {K}}{\|H+ {K}\|_F}\widetilde{V}^\top,\;
I - \widetilde{V} \widetilde{V}^\top
\Big),
\end{equation}
where
\begin{equation*}
\widetilde{V} = (V +  {V}_{p}) \big(I_r + {V}_{p}^\top {V}_{p} \big)^{-1/2}.
\end{equation*}
This retraction, together with the local rigidity condition, ensures that the negative Riemannian gradient direction remains a local descent direction of the objective function, thereby supporting convergence to critical points. 
In particular, according to (\ref{eq:retraction_update}), the update rule of Riemannian gradient-based optimization
\begin{equation}
(X,G) \leftarrow{} R_{(X,G)} \Big( - \theta \, \mathrm{grad} \bar{f}(X,G) \Big)
\end{equation}
could be computed easily by replacing $K$ and $V_p$ in Eq. \eqref{eq:retraction_nmps} by $-\theta \bar{K}$ and $-\theta \bar{V}_p$ respectively, where $\theta>0$ is the learning rate and $(\bar{K},\bar{V}_p)$ is the representation of $\mathrm{grad} \bar{f}(X,G)$ defined in Eq. \eqref{eq:Mh_gradient}.

To efficiently solve the optimization problem of UMPS model, 
we apply the space-decoupling method described above to tackle the subproblem (\ref{eq:nmps_subproblem_compact}) arising during UMPS training. 
By rewriting the target fourth-order tensor $A^{(k,k+1)}$ in a matrix form that satisfies 
$\|A_{\langle 2 \rangle}^{(k,k+1)}\|_F = 1$ and $\mathrm{rank}(A_{\langle 2 \rangle}^{(k,k+1)}) \le r$, 
the original problem is equivalently transformed into a Riemannian optimization problem on the manifold $\mathcal{M}_h$, 
ensuring that structural constraints are automatically satisfied throughout the iteration process. It is worth noting that (\ref{eq:nmps_subproblem_compact}) does not directly compute the optimal solution for $A_{\langle 2 \rangle}^{(k,k+1)}$; rather, it aims to progressively approach optimality via a Riemannian gradient-based method.

The algorithm for solving the UMPS problem using the space-decoupling method 
is presented in Algorithm~\ref{alg:nmps_training}. Notably, the algorithm preserves 
the mixed-canonical form of the MPS throughout the procedure. When performing the SVD on the updated 
$A_{\langle 2 \rangle}^{(k,k+1)}$, the truncation step is omitted, 
thanks to the maintenance of the low-rank structure. 
This ensures that no energy is lost in $\mathrm{\Psi}(\mathbf{v})$ during the update. 
Let $n=\max\{n_1, \ldots,n_d\}$. Table \ref{tab:complexity} summarizes the computational complexity of the basic operations involved in Algorithm 1. As shown in Table \ref{tab:complexity}, when the training dataset is large, the dominant computational cost is $O(|\mathcal{T}|dn^2 r_{\max}^2)$, arising from the computation of the Euclidean gradient $\nabla f(X)$. As a result, the total time complexity of Algorithm \ref{alg:nmps_training} is $O\left(l_{\max}d(|\mathcal{T}|dn^2 r_{\max}^2+n^2  r_{\max}^3)\right)$.

\begin{algorithm}
\caption{UMPS training via space-decoupling(UMPS-SD)}
\label{alg:nmps_training}
\begin{algorithmic}[1]
\REQUIRE{Training dataset $\mathcal{T} = \{\mathbf{v}^1, \dots, \mathbf{v}^{|\mathcal{T}|}\}$, maximum bond dimension $r_{\max}$, learning rate $\theta$, maximum number of loops $l_{\max}$.} 
\ENSURE{Optimized UMPS cores $A^{(1)}, A^{(2)}, \dots, A^{(d)}$ and target distribution $P(\mathbf{v}) = |\Psi(\mathbf{v})|^2$.}
\STATE{Initialize cores with right-canonical $A^{(2)}, \ldots, A^{(d)}$ and $\|A^{(1)}\|_F = 1$.}
\FOR{$i = 1$ \textbf{to} $l_{\max}$}
    \FOR{$k = 1$ \textbf{to} $d-1$ \textbf{and back to} $1$}
        \STATE \hspace{0.1cm} Parameterize $A^{(k,k+1)}$ as matrix $X = A_{\langle 2 \rangle}^{(k,k+1)}$.
        \STATE \hspace{0.1cm} Compute SVD: $X = H \Sigma V^\top$.
        \STATE \hspace{0.1cm} Normalize $H \gets H / \|H\|_F$.
        \STATE \hspace{0.1cm} Set $G = I - VV^\top$, obtaining $(X,G) \in \mathcal{M}_{h}$.
        \STATE \hspace{0.1cm} Compute the representation $(\bar{K}, \bar{V}_{p})$ of Riemannian \\ \hspace{0.5cm} gradient $\mathrm{grad} \bar{f}(X,G)$ via Eq. \eqref{eq:Mh_gradient}.
        \STATE \hspace{0.1cm} Update $(X,G)$ via retraction: \\
               \hspace{0.5cm} $X \gets \frac{H-\theta \bar{K}}{\|H-\theta \bar{K}\|_F}\widetilde{V}^\top$,  \quad $G \gets I - \widetilde{V} \widetilde{V}^\top$,  \\
               \hspace{0.5cm} where $\widetilde{V} = (V - \theta \bar{V}_{p})(I_r + \theta^{2} \bar{V}_{p}^\top \bar{V}_{p})^{-1/2}$.
        \STATE \hspace{0.1cm} Compute SVD: $X = U_k \Sigma_k V_k^\top$.
        \IF{$k$ is increasing ($k = 1$ \textbf{to} $d-1$)}
            \STATE Reshape $U_k$ to $A^{(k)}$ and reshape $\Sigma_k V_{k}^\top$ to  $A^{(k+1)}$.
        \ELSE
            \STATE Reshape $U_k \Sigma_k$ to $A^{(k)}$ and reshape $V_{k}^\top$ to  $A^{(k+1)}$.
        \ENDIF
    \ENDFOR
\ENDFOR
\end{algorithmic}
\end{algorithm}

\begin{table}[t]
\caption{\label{tab:complexity} Summary of time and space complexity for the main operations and objects of Algorithm \ref{alg:nmps_training}
.}
\centering
\begin{tabular}{>{\centering\arraybackslash}m{5cm}|>{\centering\arraybackslash}m{2.5cm}}
\hline
Operation & Time complexity    \\  
\hline  \vspace{0.8mm}
Two SVDs (lines 5 and 10)  &  $O(2nr_{\max}^2+r_{\max}^3)$  \\
\hline  \vspace{0.8mm}
Compute $\nabla f(X)$ via (\ref{eq:mps_gradient_simplified}) & $O(|\mathcal{T}|dn^2 r_{\max}^2)$ \\
\hline  \vspace{0.8mm}
Compute $(\bar{K}, \bar{V}_{p})$ via (\ref{eq:Mh_gradient}) (line 8) & $O(n^2r_{\max}^3) $\\
\hline  \vspace{0.8mm}
Update $X$ (line 9) &  $O(n^2r_{\max}^3))$  \\
\hline \hline
Object & Space complexity \\
\hline \vspace{0.5mm}
MPS cores $A^{(1)}, \dots, A^{(d)}$ & $dnr_{\max}^2$ \\
\hline \vspace{0.5mm}
Representation of $(X,G) \in \mathcal{M}_{h}$: $H, \Sigma, V$ & $(2n+1)r_{\max}^2$ \\
\hline  \vspace{0.5mm}
Representation of $\mathrm{grad} \bar{f}(X,G)$: $\bar{K}, \bar{V}_{p}$ &  $2nr_{\max}^2$     \\
\hline
\end{tabular}
\end{table}

\subsection{Generation and Sampling}
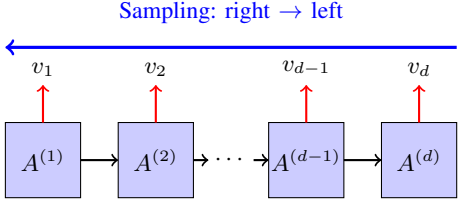
\begin{figure}[t]
\centering
\begin{tikzpicture}[scale=1, every node/.style={font=\small}]
\draw[fill=blue!20] (0,0) rectangle (1,1);
\node at (0.5,0.5) {$A^{(1)}$};
\draw[fill=blue!20] (1.5,0) rectangle (2.5,1);
\node at (2,0.5) {$A^{(2)}$};
\node at (3,0.5) {$\cdots$};
\draw[fill=blue!20] (3.5,0) rectangle (4.5,1);
\node at (4,0.5) {$A^{(d-1)}$};
\draw[fill=blue!20] (5,0) rectangle (6,1);
\node at (5.5,0.5) {$A^{(d)}$};
\draw[->, thick] (1,0.5) -- (1.5,0.5);
\draw[->, thick] (2.5,0.5) -- (2.7,0.5);
\draw[->, thick] (3.3,0.5) -- (3.5,0.5);
\draw[->, thick] (4.5,0.5) -- (5,0.5);
\draw[->, thick, red] (0.5,1) -- (0.5,1.5);
\node[above] at (0.5,1.5) {$v_1$};
\draw[->, thick, red] (2,1) -- (2,1.5);
\node[above] at (2,1.5) {$v_2$};
\draw[->, thick, red] (4,1) -- (4,1.5);
\node[above] at (4,1.5) {$v_{d-1}$};
\draw[->, thick, red] (5.5,1) -- (5.5,1.5);
\node[above] at (5.5,1.5) {$v_d$};
\draw[->, thick, blue, very thick] (6,2) -- (0,2);
\node[above, blue] at (3,2.2) {Sampling: right $\to$ left};

\end{tikzpicture}
\caption{Schematic of sequential sample generation from a UMPS. All tensors except $A^{(d)}$ are gauged to be left-canonical, Each tensor $A^{(i)}$ generates one bit $v_i$ conditionally on the bits to its right.}
\label{fig:mps_sampling_right_to_left_explicit}
\end{figure}

Once the generative model is well-trained, the next step is to generate new samples according to the learned probability.
In many popular generative models, especially energy-based models such as restricted Boltzmann machines (RBM)~\cite{fischer2012introduction}, generating new samples often requires running Markov chain Monte Carlo (MCMC) from an initial configuration due to the intractability of the partition function.  
One advantage of MPS-based generative modeling is that the partition function can be exactly computed with linear complexity in the system size~\cite{han2018unsupervised}. This allows a direct and efficient sampling method, which generates a sample bit by bit from one end of the UMPS to the other (see Fig.~\ref{fig:mps_sampling_right_to_left_explicit}). 

For the proposed generative model based on UMPS, the probability is given by
$P(\mathbf{v}) = |\Psi(\mathbf{v};A^{(1)}, \cdots, A^{(d)})|^2. $
The detailed sampling procedure is as follows:
\begin{enumerate}[label=(\roman*)]
    \item Start from one end, e.g., the $d$th bit. Sample $v_d$ from the marginal probability
    \[
        P(v_d) = \sum_{v_1,\dots,v_{d-1}} P(v_1,\dots,v_d) = \|A^{(d)}(v_d)\|_F^2,
    \]
    where all tensors except $A^{(d)}$ are gauged to be left-canonical. Note that inverse sampling is similar, but all tensors except $A^{(1)}$ are gauged to be right-canonical. 
    
    \item Given $v_d$, move to the $(d-1)$th bit. More generally, given the values $v_k, v_{k+1},\dots,v_d$, the $(k-1)$th bit is sampled according to the conditional probability
    \[
        P(v_{k-1}\,|\,v_k, \dots, v_d) = \frac{P(v_{k-1},v_k,\dots,v_d)}{P(v_k, \dots, v_d)}.
    \]
    
    \item Thanks to the canonical condition, the marginal probability can be efficiently computed as
    \[
        P(v_k, \dots, v_d) = \|A^{(k)}(v_k) \cdots A^{(d)}(v_d)\|_F^2. \]

    \item By iteratively multiplying $A^{(k-1)}(v_{k-1})$ from the left and computing the squared norm, one obtains $P(v_{k-1},v_k,\dots,v_d)$, and samples $v_{k-1}$ accordingly.
\end{enumerate}

Repeating this process from right to left produces a sample that strictly obeys the probability distribution defined by the MPS. This approach is also applicable to inference tasks when part of the bits are given. Even if unknown bits are interleaved among known bits, marginal probabilities remain tractable due to efficient contraction of ladder-shaped tensor networks~\cite{orus2014practical}. 

\section{Numerical experiments}
To validate the reliability of the UMPS-SD algorithm, we apply it to several commonly used datasets, including the Bars and Stripes (BAS) dataset\footnote{\url{https://pennylane.ai/datasets/bars-and-stripes}.} and the EMNIST dataset\footnote{\url{https://www.nist.gov/itl/products-and-services/emnist-dataset}.}.
In all experiments, each grayscale or binary image of size $n \times n$ is reshaped into a one-dimensional vector of length $n^2$
using MATLAB's reshape function, where the columns of the original image are stacked sequentially. This preprocessing step ensures compatibility with the MPS-based generative model, as illustrated in Fig.~\ref{fig:flatten4}.
To ensure a fair comparison, all baseline experiments are performed using the same initial MPS and identical hyperparameter settings.

All experiments were conducted using Matlab R2022a on a computer with an Intel(R) Core(TM) i7-14700F @2.10GHz CPU, 16 GB of RAM and an NVIDIA GeForce RTX 4060Ti GPU with 8 GB of VRAM. The source code used in this work is publicly available at: \url{https://github.com/haotong-Duan/UnitaryMPS-SpaceDecoupling}.

\begin{figure}[t]
\centering
\subfloat[]{%
    \includegraphics[width=0.43\linewidth]{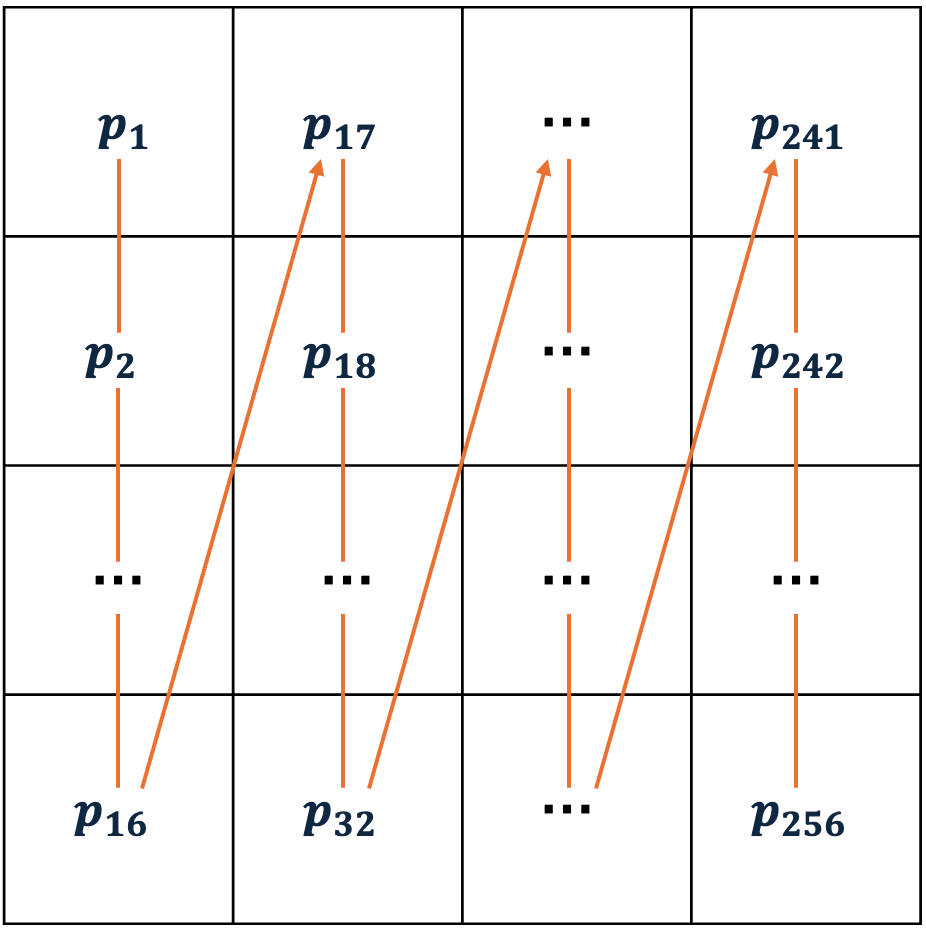}%
    \label{fig:flatten4}
}
\hfil
\subfloat[]{%
    \includegraphics[width=0.47\linewidth]{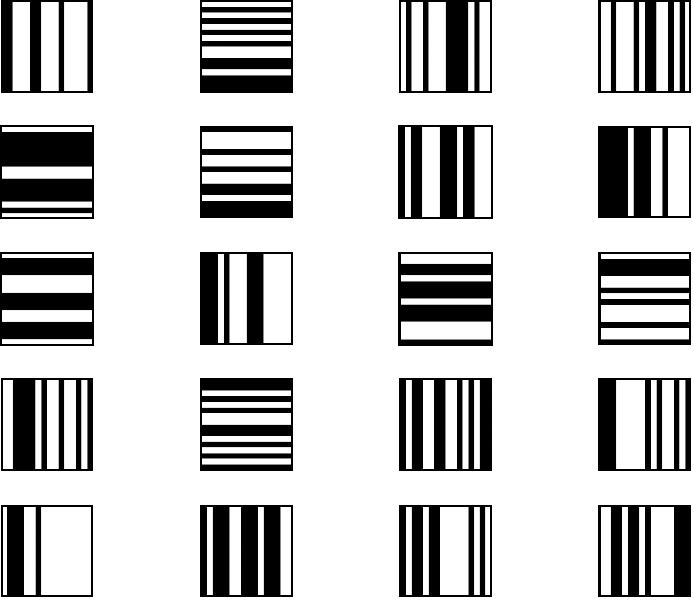}%
    \label{fig:bars_stripes_16}
}
\caption{(a) Illustration of column-major flattening of a $16\times16$ image into a 256-dimensional vector $v = (p_1, p_2, \dots, p_{16},p_{17}, \dots , p_{256})^\top$. (b) A subset of the Bars-and-Stripes dataset, with images of size $16 \times 16$.}
\label{fig:bar_and_flatten}
\end{figure}

\subsection{Bars and Stripes}
\begin{figure}[t]
\centering
\subfloat[]{%
    \includegraphics[width=0.55\linewidth]{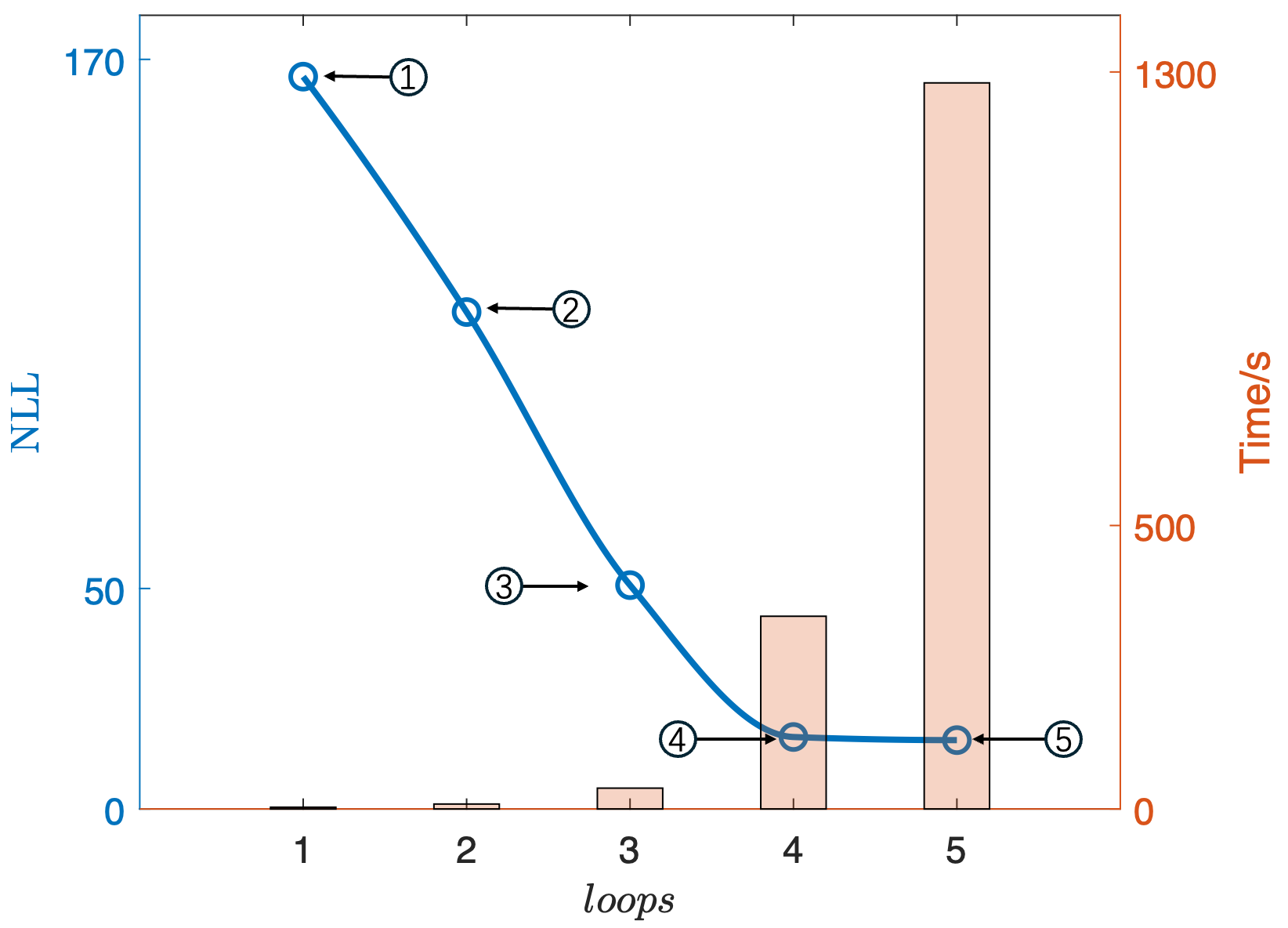}%
    \label{fig:bas_nll_loops}
}
\subfloat[]{%
    \raisebox{8pt}{\includegraphics[width=0.35\linewidth]{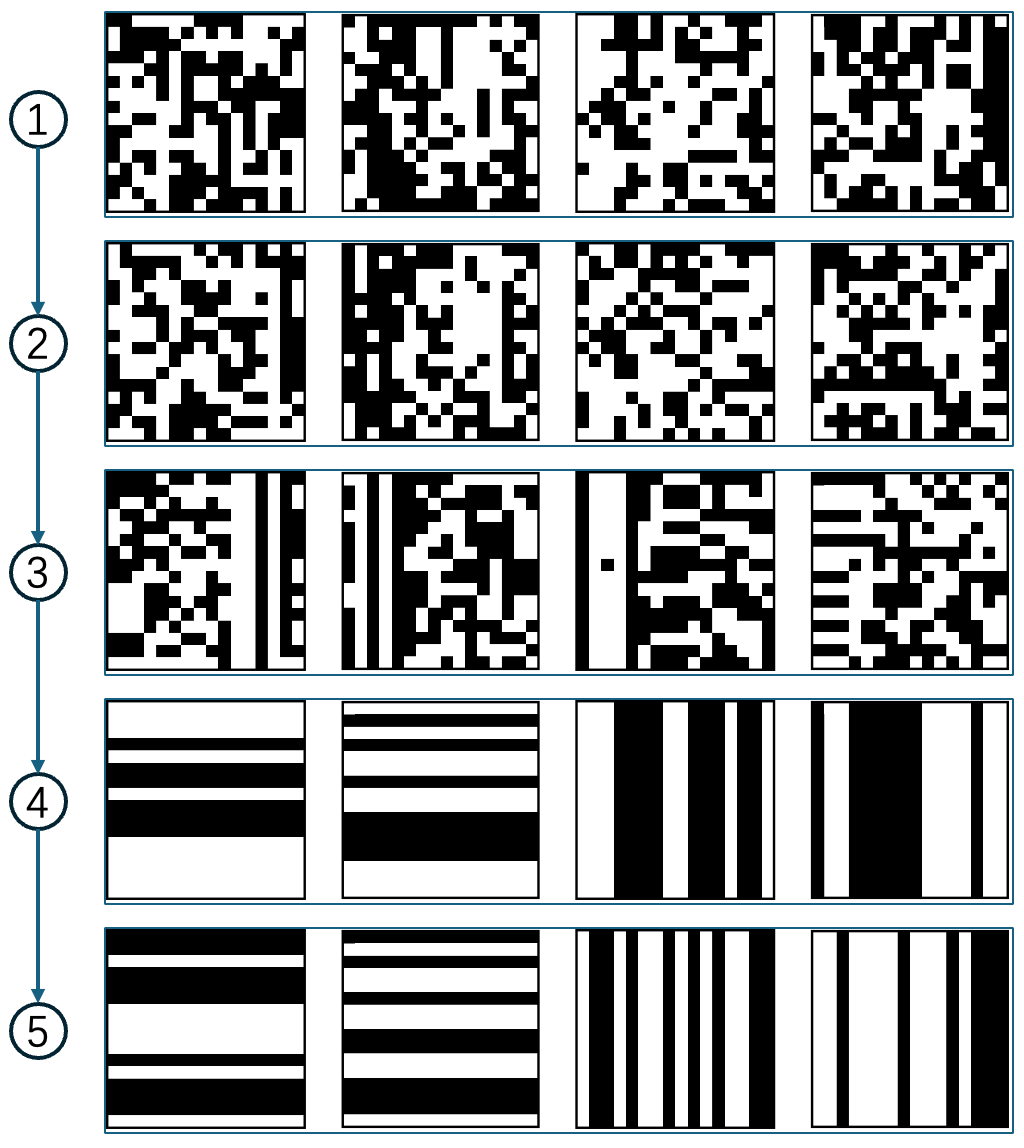}}
    \label{fig:bas_16_1_5}
}
\caption{(a) The blue curve represents the average NLL as a function of the number of loops, while the orange bars indicate the cumulative computation time required to reach the corresponding loops. (b) Images generated by the model at loops from 1 to 5 as indicated in (a). In experiments above, we set $r_{\rm max}=500$, $|\mathcal{T}|=400$, and the learning rate to $0.007$.}
\label{fig:bas_nll_loops_generate}
\end{figure}

\begin{table}[t]
\caption{\label{tab:bond_dims} Under the experimental conditions shown in Fig.~\ref{fig:bas_nll_loops_generate}, as the number of loops increases, the table shows the average bond dimension $r_{\rm mean}$ and the largest bond dimension $r_{\rm max}$ in the UMPS updated up to this step.}
\centering
\begin{tabular}{c|c|c|c|c|c}
\hline
$loops$ & 1 & 2 & 3 & 4 & 5 \\  
\hline
$r_{\rm mean}$ & 7.8945 & 31.1133 & 122.4883 & 470.7852 & 470.7852 \\
$r_{\rm max}$  & 8 & 32 & 128 & 500 & 500 \\
\hline
\end{tabular}
\end{table}

\begin{figure}[t]
\centering
\subfloat[]{%
    \includegraphics[width=0.48\linewidth]{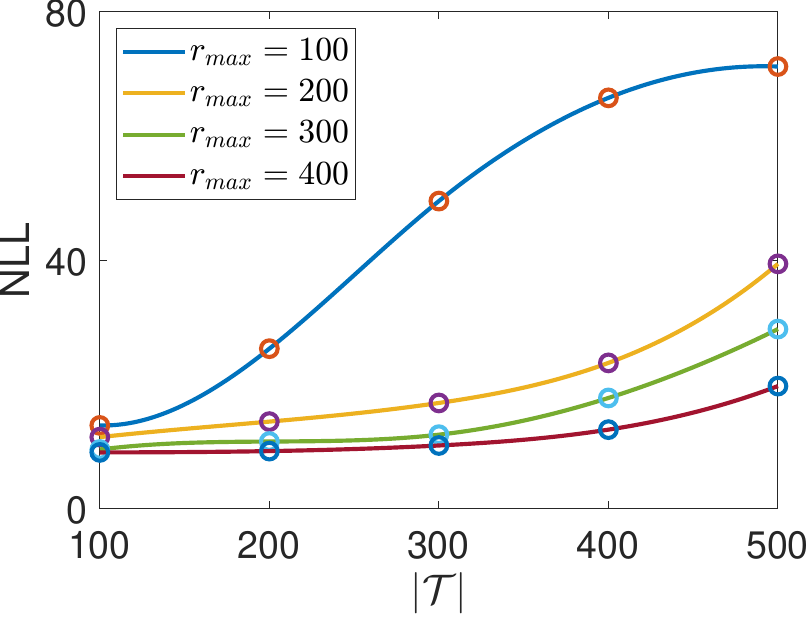}%
    \label{fig:bas_T_nll}
}
\hfill
\subfloat[]{%
    \includegraphics[width=0.48\linewidth]{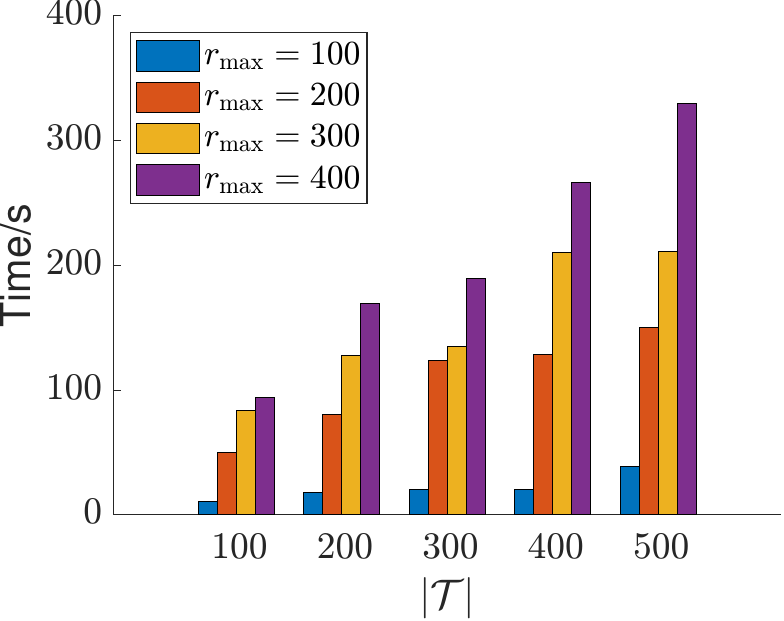}%
    \label{fig:bas_times}
}
\caption{(a) Average NLL as a function of the dataset size under different $r_{\rm max}$. For each curve, the circular markers from left to right correspond to the NLL values obtained at $|\mathcal{T}| = 100, 200, 300, 400,$ and $500$, respectively. (b) The average computation time per experiment shown in (a).}
\label{fig:bas_nll_time_16}
\end{figure}

To evaluate the performance of the proposed UMPS-SD algorithm, we conduct experiments on the \textit{Bars and Stripes} (BAS) dataset. The BAS dataset used in this work consists of all possible horizontal-bar and vertical-bar binary images of size $16\times16$. Each image contains $256$ binary pixels, and the complete dataset includes $2\times2^{16}-2 = 131070$ distinct stripe patterns. 

Based on the BAS dataset described above, we examine the decrease of the NLL with respect to the number of loops under the UMPS--SD algorithm. As shown in Fig.~\ref{fig:bas_nll_loops}, the NLL exhibits a rapid decline at the beginning and then gradually levels off. We generate samples at each loop, and in the early stages the synthesized images contain substantial noise. However, once $loops \ge 4$, the model is able to produce valid images (Fig.~\ref{fig:bas_16_1_5}) that are consistent with the characteristics of the BAS dataset. In fact, although the overall computational time increases significantly as the number of loops grows, the average bond dimension $r_{\rm mean}$ of the MPS tends to stabilize, and the bond dimensions do not exceed $r_{\rm max}$ (Table~\ref{tab:bond_dims}), demonstrating that the algorithm effectively preserves the low-rank structure while the model provides sufficient representational capacity to generate high-quality images.

To further investigate how the NLL varies with the dataset size under different maximum bond dimensions $r_{\max}$, as well as the time consumed per training update. Specifically, for multiple choices of $r_{\max}$, we train the UMPS model while gradually increasing $|\mathcal{T}|$, and record the corresponding NLL values. This allows us to examine the modeling capability of the MPS under different expressive constraints. To reduce computational cost, we set $l_{\max}=4$ in the experiments above. At this setting, the algorithm does not convergence completely, yet it already demonstrates satisfactory performance. The results are shown in Fig.~\ref{fig:bas_T_nll}.
Under the same experimental settings, we further measure the computational cost of each training update. The results, presented in Fig.~\ref{fig:bas_times}, illustrate how different values of $r_{\max}$ influence the convergence efficiency of the algorithm.

As shown in Fig.~\ref{fig:bas_nll_time_16}, when $|\mathcal{T}|>r_{\rm max}$, the NLL exhibits a noticeable increasing trend, indicating that the model's capacity is insufficient. To achieve better results, a larger $r_{\rm max}$ should be employed. By observing the update times of the algorithm in each experiment, we find that even in the most complex model, the UMPS-SD algorithm can reach an acceptable performance within 350 seconds.

\subsection{EMNIST dataset}
To systematically evaluate the performance of the proposed UMPS\mbox{-}SD algorithm, 
we conduct a comparative study against the baseline method introduced by Han \emph{et al.} ~\cite{han2018unsupervised}. 
The experiments assess two primary aspects---convergence efficiency and generative quality.

\begin{enumerate}[label={}, leftmargin=0pt]
    \item \textbf{Model parameters:}
    The maximum bond dimension $r_{\rm max}$ of MPS serves as a key hyperparameter. We evaluate the algorithm under 
    several choices of the maximum bond dimension in order to examine its performance across 
    different model complexities.

    \item \textbf{Evaluation metrics:}
    Convergence is quantified using the negative log-likelihood on the test set, while the 
    total training time is recorded to assess computational efficiency.
\end{enumerate}

\begin{figure}[!t]
\centering
\subfloat[]{%
    \includegraphics[width=0.48\linewidth]{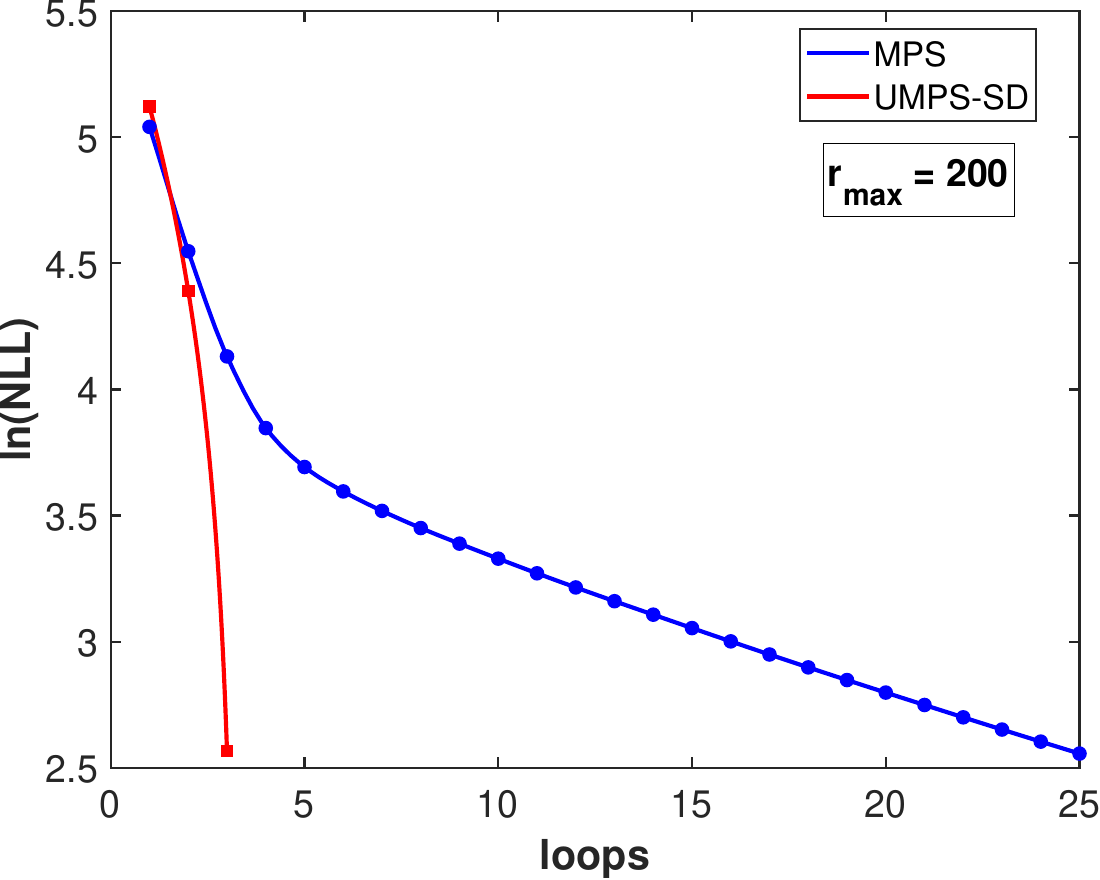}%
    \label{fig:nll_log_200}
}
\hfill
\subfloat[]{%
    \includegraphics[width=0.48\linewidth]{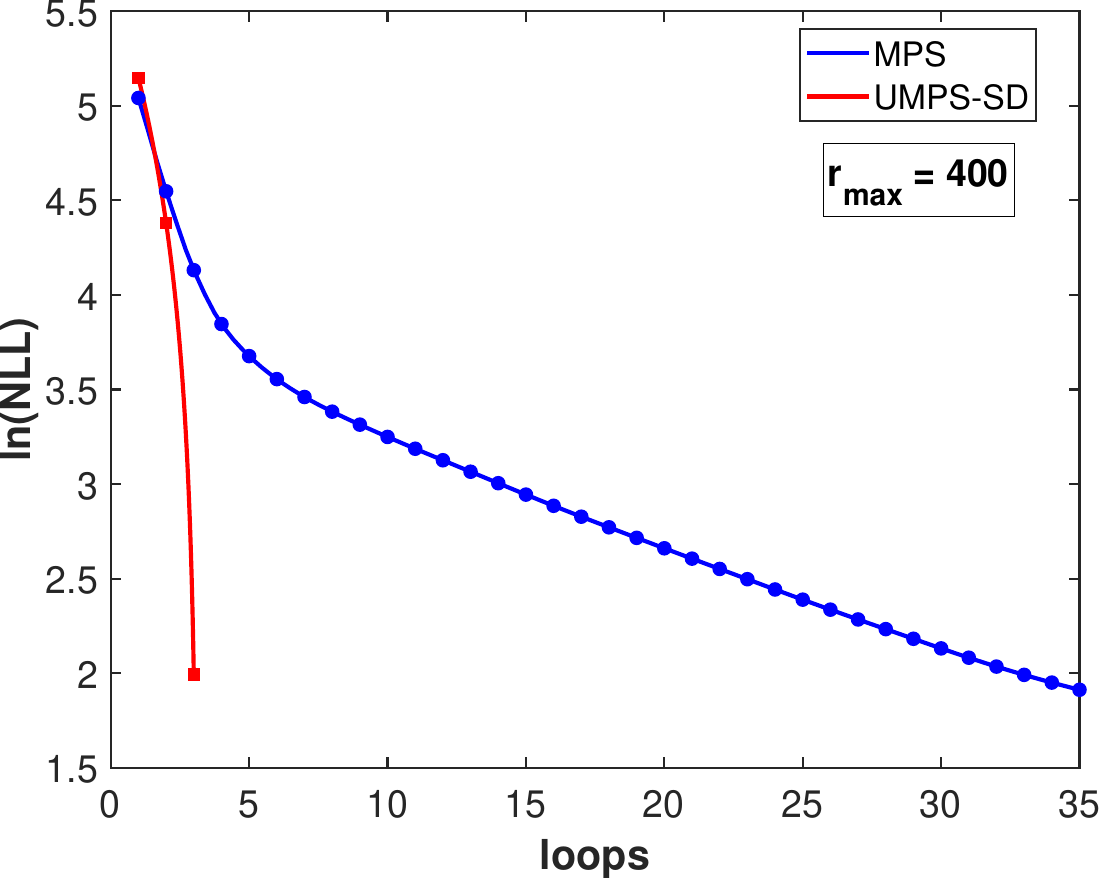}%
    \label{fig:nll_log_400}
}
\caption{Panels (a) and (b) correspond to the negative log-likelihood (NLL) convergence of the optimization algorithm of the MPS model and UMPS-SD algorithm for maximum bond dimensions $r_{\rm max}=200$ and $r_{\rm max}=400$, respectively, with a training set of $|\mathcal{T}|=100$ samples. The blue line represents the optimization algorithm of the MPS model, and the red line represents the UMPS-SD algorithm.}
\label{fig:nll_log}
\end{figure}

\subsubsection{Convergence efficiency}
To evaluate the convergence efficiency of the Algorithm \ref{alg:nmps_training}, we conducted experiments under identical settings and preconditions. The maximum bond dimension, which governs the model complexity, serves as the key hyperparameter. Two sets of experiments were performed with different maximum bond dimensions of. The convergence behavior of the negative log-likelihood was monitored to assess the performance of each method, as shown in Fig.~\ref{fig:nll_log}.

The experimental results clearly demonstrate that the proposed UMPS\text{-}SD algorithm exhibits significantly faster convergence than the standard MPS-based method. UMPS\text{-}SD displays a steep decline in the NLL during the initial phase of training, whereas the updates in the MPS model progress much more gradually. This indicates that, for the same number of training iterations, the proposed algorithm approaches the optimal solution more rapidly. To provide a more direct assessment of convergence speed, Table ~\ref{tab:nll_time} compares the temporal evolution of the objective function for both algorithms under identical experimental settings.

\begin{table}[t]
\caption{\label{tab:nll_time} Under the same experimental conditions as Fig. \ref{fig:nll_log_200}, NLL and training time (s) for MPS and UMPS-SD models with learning rate $1\times 10^{-3}$. The time metric indicates the cumulative time required to reach this NLL value.}
\centering
\begin{tabular}{cccccc}
\hline
\textbf{$loops$} & \multicolumn{2}{c}{\textbf{MPS}} & \multicolumn{2}{c}{\textbf{UMPS-SD}} \\ 
\cline{2-5}
 & \textbf{NLL} & \textbf{Time/s} & \textbf{NLL} & \textbf{Time/s} \\ \hline
1  & 154.74 & 1.30    & 167.70 & 1.68    \\
2  & 94.48  & 3.14    & 80.69  & 5.61    \\
3  & 62.25  & 7.86    & 13.01  & 30.25   \\
4  & 46.85  & 26.08   & -      & -       \\
5  & 40.14  & 53.61   & -      & -       \\
6  & 36.44  & 84.13   & -      & -       \\
\vdots & \vdots & \vdots & - & - \\
23 & 14.17  & 754.18  & -      & -       \\
24 & 13.51  & 794.47  & -      & -       \\
25 & 12.88  & 831.44  & -      & -       \\
\hline
\end{tabular}
\end{table}

The comparison between the two algorithms demonstrates that the proposed method achieves a substantially faster convergence. Specifically, it reduces the NLL from 167.70 to 13.01 within only three training loops, whereas the MPS method attains an NLL of 62.25 after the same number of loops and requires 25 loops to reach a comparable accuracy (NLL = 12.88). Although each iteration of the proposed algorithm incurs slightly higher computational cost, its overall convergence is markedly faster, up to 27$\times$ more efficient than the original algorithm. Moreover, with an appropriately chosen learning rate, the proposed approach attains both higher convergence accuracy and faster descent than the standard MPS algorithm, as shown in Table \ref{tab:nll_lr}. Collectively, these results indicate that the proposed method follows a more efficient optimization trajectory, significantly accelerating the training process and outperforming conventional approaches in terms of both convergence speed and computational efficiency.

\begin{table}[t]
\caption{\label{tab:nll_lr} Under the same experimental conditions as Fig.~\ref{fig:nll_log_200}, from (1) to (3), the learning rates are set to $2.5\times 10^{-4}$, $5\times 10^{-4}$, and $2\times 10^{-3}$, respectively.}
\centering
\begin{tabular}{c|cc|cc|cc}
\hline
\textbf{$loops$} 
& \multicolumn{2}{c|}{\textbf{UMPS-SD(1)}} 
& \multicolumn{2}{c|}{\textbf{UMPS-SD(2)}} 
& \multicolumn{2}{c}{\textbf{UMPS-SD(3)}} 
\\ \cline{2-7}
& \textbf{NLL} & \textbf{Time/s} 
& \textbf{NLL} & \textbf{Time/s} 
& \textbf{NLL} & \textbf{Time/s} 
\\ \hline
1  & 171.99  & 1.62   & 164.41 & 1.59   & 165.76 & 1.62   \\
2  & 79.90   & 5.59   & 81.14  & 5.48   & 79.97  & 5.55   \\
3  & 7.35    & 30.73  & 6.568  & 30.23  & 12.30  & 30.11  \\
\hline
\end{tabular}
\end{table}

\subsubsection{Generalization ability}
\begin{figure}[t]
\centering
\subfloat[]{%
    \fbox{\includegraphics[width=0.421\linewidth]{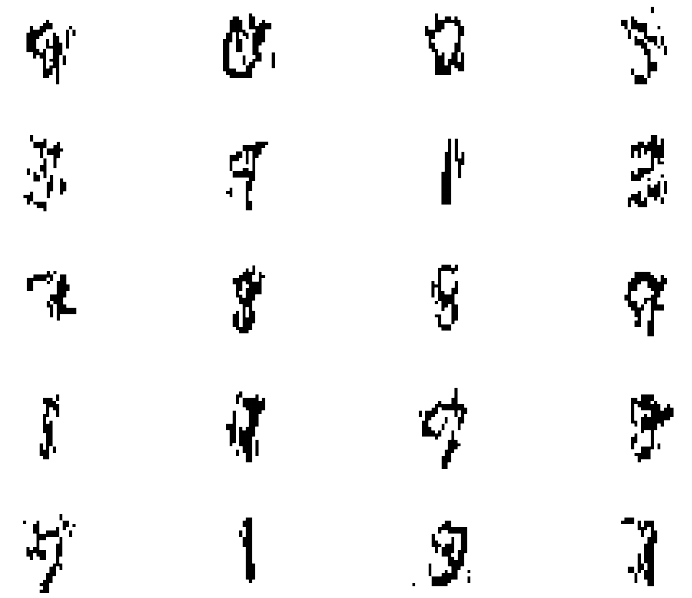}}%
    \label{fig:generate_400_mps}
}
\hfill
\subfloat[]{%
    \fbox{\includegraphics[width=0.421\linewidth]{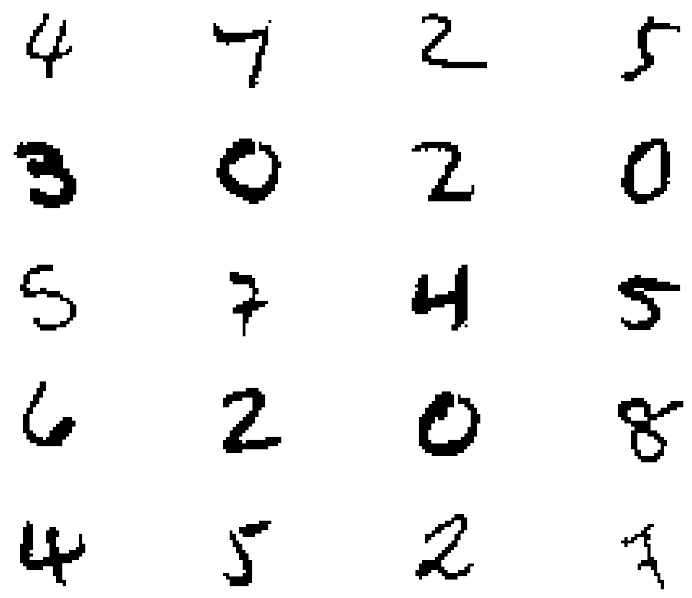}}%
    \label{fig:generate_400}
}
\caption{(a) Generated from the MPS model with a maximum bond dimension $r_{\max}=400$, trained on a subset of $|\mathcal{T}|=300$ samples randomly selected from the EMNIST dataset for $l_{\max} = 25$. (b) Generated by the proposed UMPS model with $l_{\max} = 4$, with all other settings kept identical to (a). }
\label{fig:two_figs}
\end{figure}

Generative modeling constitutes one of the primary tasks in machine learning. To evaluate the generative capability of our method, Fig. \ref{fig:generate_400_mps} shows samples generated by the MPS model. For comparison, we then use the trained UMPS model to directly generate samples, as illustrated in Fig. \ref{fig:generate_400}.
In the reconstruction experiment, the pixels in the right half of each image (i.e., positions 393 to 784) are provided as input, and the trained UMPS model is required to infer and complete the missing left half. 
We further compare the reconstructed results with those obtained using the conventional MPS model under identical settings, as presented in Fig. \ref{fig:reconstruct_contrast}.

\begin{figure}[t]
\centering
\subfloat[]{%
    \fbox{\includegraphics[width=0.421\linewidth]{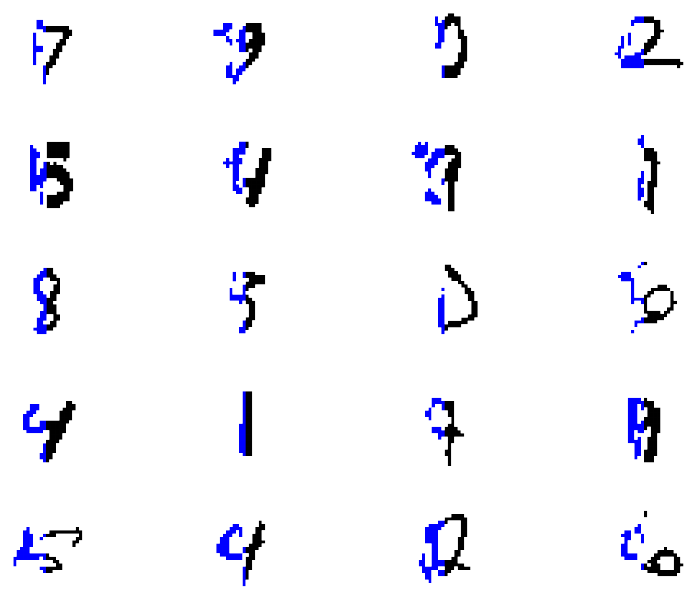}}%
    \label{fig:mps_20025}
}
\hfill
\subfloat[]{%
    \fbox{\includegraphics[width=0.421\linewidth]{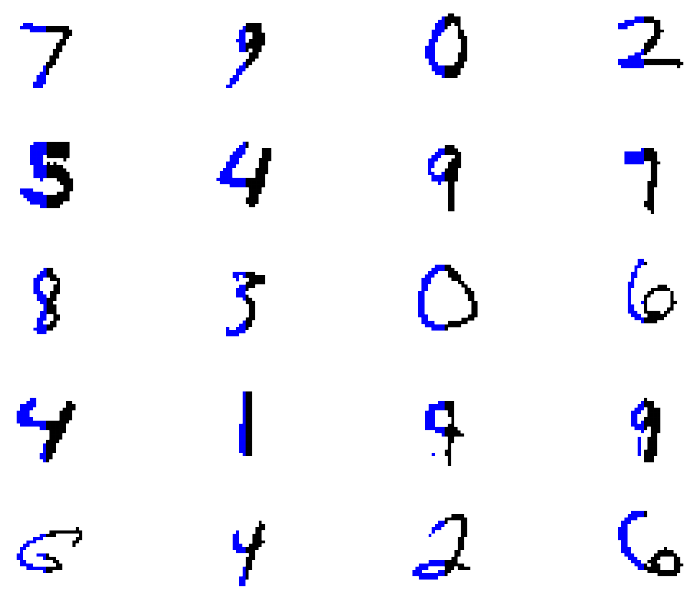}}%
    \label{fig:nmps_2003}
}
\caption{Image reconstruction from half of the images, which are selected from the training set. For both experiments, the maximum bond dimension is set to $r_{\rm max} = 200$ and $|\mathcal{T}| = 100$. The update details are summarized in Table~\ref{tab:nll_time}. The black areas in the image are the given pixels, and the blue areas are the pixels completed by the model. (a) The image reconstructed by the MPS model; (b) The image reconstructed by the UMPS model.}
\label{fig:reconstruct_contrast}
\end{figure}

From Figs. \ref{fig:two_figs} and \ref{fig:reconstruct_contrast}, it is evident that the MPS model performs well only on a small subset of images, while producing noticeable noise and frequent reconstruction errors on most others. In contrast, although the UMPS model occasionally exhibits minor reconstruction errors in a few individual cases, it consistently demonstrates superior detail recovery and generates significantly less noise in the completed regions. For instance, for the digits ``4" and ``5" in Fig. \ref{fig:reconstruct_contrast}, the left halves completed by the UMPS model connect more naturally with the original right halves, and the stroke contours are clearer and more complete. In comparison, the MPS results may display distorted or broken strokes. For simpler digits such as ``1", UMPS model achieves highly accurate completion, whereas the MPS model may erroneously reconstruct them into shapes resembling other digits, sometimes resulting in nearly unrecognizable forms. For the image located in the fourth row and third column of Fig. \ref{fig:nmps_2003}, the missing region is reconstructed as a digit “9”, whereas the original image corresponds to a “7”. However, since the original sample does not follow a standard writing style for the digit “7”, this completion can still be considered reasonable. 

\begin{figure*}
\includegraphics[width=\linewidth]{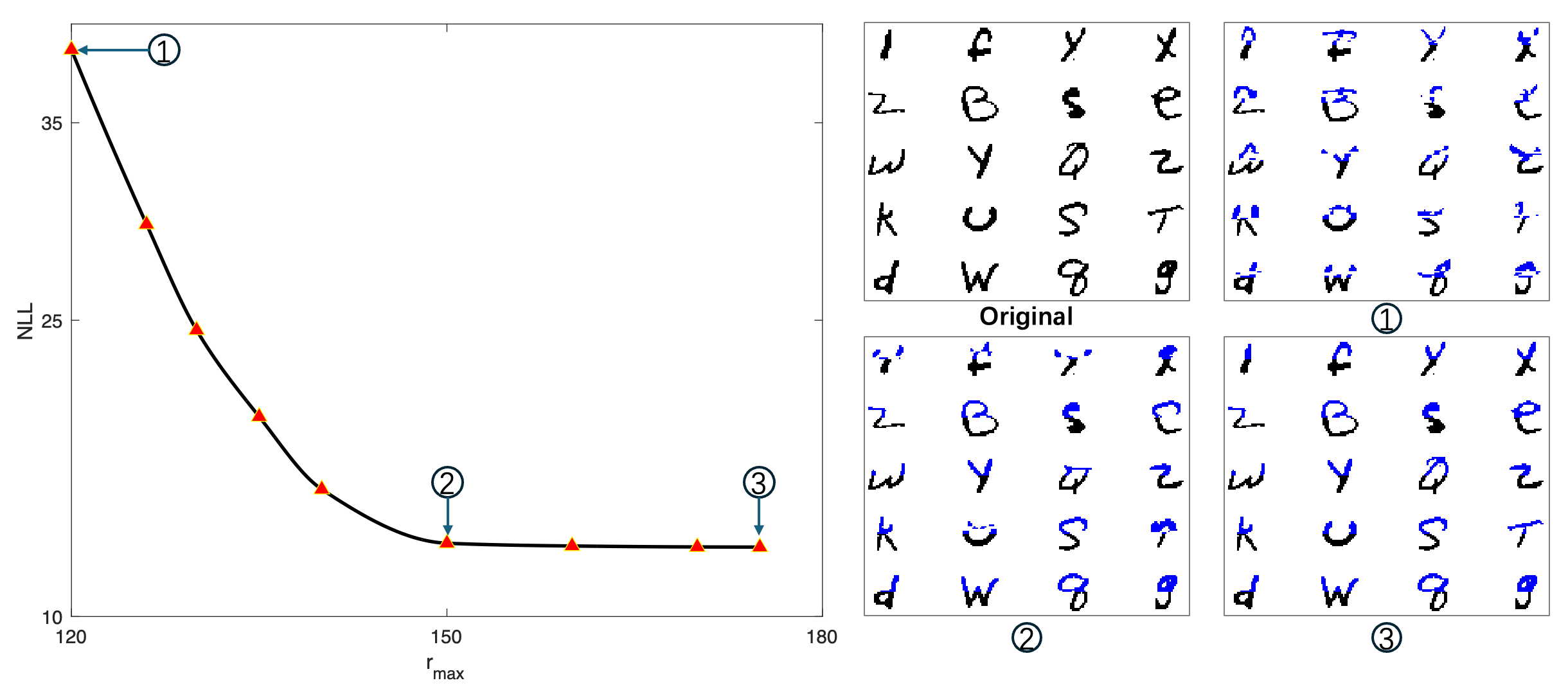}
\caption{\label{fig:descent}The curves illustrate how the NLL varies with the bond dimension $r_{\rm max}$, where the training set $\mathcal{T}$ is selected from the EMNIST-Letters dataset with a fixed size of $|\mathcal{T}| = 150$. All marked points correspond to the setting of $l_{\max}=4$. In the images, black pixels denote the given observations, while the blue pixels are reconstructed by the UMPS model. The figure presents the reconstruction results obtained for $r_{\rm max} = 120$, $150$, and $175$, respectively.
}
\end{figure*}

The proposed model also performs well on the more diverse EMNIST-Letters dataset. To highlight the efficiency of the algorithm, we restrict the training procedure to only $l_{\max}=4$. In our experiments, with a fixed bond dimension $r_{\rm max}$, the objective function (NLL) decreases rapidly. We observe that when $r_{\rm max}>|\mathcal{T}|$, the UMPS-SD algorithm reaches the desired optimization regime within a very short time. Moreover, as $r_{\rm max}$ increases, the improvement in NLL becomes marginal. In contrast, when $r_{\rm max}$ is insufficient, the model suffers from limited representational capacity, leading to inferior reconstruction quality in certain samples. These effects are clearly illustrated in Fig. \ref{fig:descent}.

\section{Summary and outlook}
Tensor-network methods have had a profound impact on machine learning. However, conventional models such as MPS often suffer from slow updates and optimization difficulties. Motivated by these limitations, we propose an efficient UMPS update algorithm based on Riemannian optimization, which mitigates oscillations among equivalent optimal solutions during training and thereby significantly accelerates convergence.


Similar to optimization methods in Euclidean spaces, Riemannian optimization also requires selecting appropriate learning rates during gradient descent \cite{sakai2024general}. However, many machine learning algorithms struggle to achieve adaptive learning rate selection \cite{nar2018step}, particularly when the dataset is large, since the Armijo-type line search incurs substantial computational overhead. Inspired by Adam and Adagrad on Riemannian manifolds \cite{becigneul2018riemannian}, one may design adaptive learning-rate schemes that maintain first- and second-order moment estimates of gradients in the tangent space, assigning individualized learning rates to each parameter component on the manifold. Moreover, when stochastic mini-batches are used to estimate gradients \cite{mahoney2011randomized}, the variance of stochastic gradients on manifolds can severely affect convergence. Variance-reduction techniques \cite{zhang2016riemannian} provide a promising remedy for stabilizing and accelerating convergence in such settings.

At present, the UMPS model is applicable only to binarized images and cannot directly handle RGB images, mainly due to the limited expressive power of one-dimensional chain-structured MPS, whose entanglement structure is relatively simple \cite{vidal2003efficient}. Future work may explore more expressive two-dimensional tensor networks, such as projected entangled pair states (PEPS) \cite{orus2014practical,verstraete2006criticality}. Although PEPS inherently involves high computational complexity during optimization, recent advances such as variational Monte Carlo (VMC) methods for PEPS offer a promising alternative \cite{wu2025algorithms}. Furthermore, analyzing the impact of gauge freedom on variational optimization and adopting standardized gauge-fixing strategies can effectively suppress artificially low variational energies \cite{tang2025gauging}. Developing efficient approximate-contraction schemes and Riemannian-optimization strategies tailored to higher-order tensor networks represents a valuable and forward-looking research direction.

\section*{Acknowledgments}
The work of Z. Chen was partially supported by National Natural Science Foundation of China (No. 11701132), Natural Science Foundation of Zhejiang Province (No. LY22A010012) and Natural Science Foundation of Xinjiang Uygur Autonomous Region (No. 2024D01A09). The work of N. Wong was supported in part by the Theme-based Research Scheme (TRS) project T45-701/22-R of the Research Grants Council of Hong Kong, and in part by the AVNET-HKU Emerging Microelectronics and Ubiquitous Systems (EMUS) Lab.

\appendix
\section{Proof that manifolds $\mathcal{S}_{m\times n}$ and $\mathcal{M}_k$ are transversely intersecting}
\label{transverse}
Let $\mathcal{S}_{m\times n}$ and $\mathcal{M}_k$ be the unit sphere manifold and the fixed-rank-$k$ manifold as defined in (\ref{unit_sphere_manifold}) and (\ref{fix_rank_manifold}), respectively. Then
$\mathcal{S}_{m\times n}$ and $\mathcal{M}_k$ intersect transversely.

\textbf{Proof.}  
To prove the theorem, it suffices to show that for any $X \in \mathcal{S}_{m\times n} \cap \mathcal{M}_k$, the conclusion
$ T_X \mathcal{S}_{m\times n} + T_X \mathcal{M}_k = \mathbb{R}^{m \times n}$ holds, which is equivalent to 
$ (T_X \mathcal{S}_{m\times n})^\perp \cap (T_X \mathcal{M}_k)^\perp = \{0\}.$

For the unit sphere manifold $\mathcal{S}_{m\times n}$, we have
\begin{equation*}
(T_X \mathcal{S}_{m\times n})^\perp = \{ t X \mid t \in \mathbb{R} \}.
\end{equation*}
For any $X \in \mathcal{M}_k$, let $X = U V^\top$ be a full-rank factorization, where $U \in \mathbb{R}^{m \times k}, V \in \mathbb{R}^{n \times k}$. It follows that
\begin{equation*}
T_X \mathcal{M}_k = \{ \dot{U} V^\top + U \dot{V}^\top \mid 
\dot{U} \in \mathbb{R}^{m \times k}, \dot{V} \in \mathbb{R}^{n \times k} \}.
\end{equation*}

Let $W \in (T_X \mathcal{M}_k)^\perp$, i.e., 
\[
\langle W, \dot{U} V^\top + U \dot{V}^\top \rangle = 0, \quad
\forall \dot{U}\in \mathbb{R}^{m \times k}, \dot{V} \in \mathbb{R}^{n \times k}.
\]
This implies
\begin{align*}
\langle W, \dot{U} V^\top \rangle &= \mathrm{tr}(W^\top \dot{U} V^\top) 
= \mathrm{tr}(V^\top W^\top \dot{U}) = \langle W V, \dot{U} \rangle, \\
\langle W, U \dot{V}^\top \rangle &= \mathrm{tr}(W^\top U \dot{V}^\top) 
= \mathrm{tr}(\dot{V} U^\top W) = \langle W^\top U, \dot{V} \rangle.
\end{align*}
Since $\dot{U}, \dot{V}$ are arbitrary, we must have
$W V = 0$ and $W^\top U = 0$.
Let $U_\perp \in \mathbb{R}^{m \times (m-k)}$ and $V_\perp \in \mathbb{R}^{n \times (n-k)}$ be the orthogonal bases of $U$ and $V$, respectively. It follows that there exists $\Phi \in \mathbb{R}^{(m-k) \times (n-k)}$ such that $W = U_\perp \Phi V_\perp^\top$.

For any $ \Gamma \in (T_X \mathcal{S}_{m\times n})^\perp \cap (T_X \mathcal{M}_k)^\perp$, we have 
\begin{equation*}
\Gamma = t X = t U V^\top = U_\perp \Phi V_\perp^\top.
\end{equation*}
The left-hand side lies in $\mathrm{col}(U)$, while the right-hand side lies in 
$\mathrm{col}(U_\perp)$. Since $U$ and $U_\perp$ are orthogonal, we have
\[
t = 0, \quad \Phi = 0 \implies \Gamma = 0.
\]
This means
$
(T_X \mathcal{S}_{m\times n})^\perp \cap (T_X \mathcal{M}_k)^\perp = \{0\},
$
which implies $T_X \mathcal{S}_{m\times n} + T_X \mathcal{M}_k = \mathbb{R}^{m \times n}.$
Therefore, $\mathcal{S}_{m\times n}$ and $\mathcal{M}_k$ intersect transversely.

\balance


 





\begin{thebibliography}{1}
\bibliographystyle{IEEEtran}

\bibitem{eisert2010colloquium}
J. Eisert, M. Cramer, and M. B. Plenio, ``Colloquium: Area laws for the entanglement entropy,'' \textit{Rev. Mod. Phys.}, vol. 82, no. 1, pp. 277--306, 2010.

\bibitem{hastings2007area}
M. B. Hastings, ``An area law for one-dimensional quantum systems,'' \textit{J. Stat. Mech.}, no. 08, p. P08024, 2007.

\bibitem{kolda2009tensor}
T. G. Kolda and B. W. Bader, ``Tensor decompositions and applications,'' \textit{SIAM Rev.}, vol. 51, no. 3, pp. 455--500, 2009.

\bibitem{oseledets2011tensor}
I. V. Oseledets, ``Tensor-train decomposition,'' \textit{SIAM J. Sci. Comput.}, vol. 33, no. 5, pp. 2295--2317, 2011.

\bibitem{chen2018tensor} 
Z. Chen, K. Batselier, J. A. Suykens, and N. Wong, ``Parallelized tensor train learning of polynomial classifiers," \textit{IEEE Trans. Neural Netw. Learn. Syst.}, vol. 29, no. 10, pp. 4621--4632, 2018.

\bibitem{liu2023tensor}
J. Liu, S. Li, J. Zhang, and P. Zhang, ``Tensor networks for unsupervised machine learning,'' \textit{Phys. Rev. E}, vol. 107, no. 1, p. L012103, 2023.

\bibitem{orus2019tensor}
R. Orús, ``Tensor networks for complex quantum systems,'' \textit{Nat. Rev. Phys.}, vol. 1, no. 9, pp. 538--550, 2019.

\bibitem{ackley1985learning}
D. H. Ackley, G. E. Hinton, and T. J. Sejnowski, ``A learning algorithm for Boltzmann machines,'' \textit{Cognitive Sci.}, vol. 9, no. 1, pp. 147--169, 1985.

\bibitem{cheng2018information}
S. Cheng, J. Chen, and L. Wang, ``Information perspective to probabilistic modeling: Boltzmann machines versus born machines,'' \textit{Entropy}, vol. 20, no. 8, p. 583, 2018.

\bibitem{goodfellow2014generative}
I. Goodfellow \textit{et al.}, ``Generative adversarial nets,'' in \textit{Adv. Neural Inf. Process. Syst.}, vol. 27, 2014.

\bibitem{stokes2019probabilistic}
J. Stokes and J. Terilla, ``Probabilistic modeling with matrix product states,'' \textit{Entropy}, vol. 21, no. 12, p. 1236, 2019.

\bibitem{cichocki2016tensor}
A. Cichocki \textit{et al.}, ``Tensor networks for dimensionality reduction and large-scale optimization: Part 1 low-rank tensor decompositions,'' \textit{Found. Trends Mach. Learn.}, vol. 9, no. 4--5, pp. 249--429, 2016.

\bibitem{han2018unsupervised}
Z.-Y. Han \textit{et al.}, ``Unsupervised generative modeling using matrix product states,'' \textit{Phys. Rev. X}, vol. 8, no. 3, p. 031012, 2018.

\bibitem{li2018shortcut}
Z. Li and P. Zhang, ``Shortcut matrix product states and its applications,'' \textit{arXiv:1812.05248}, 2018.

\bibitem{glasser2019expressive}
I. Glasser \textit{et al.}, ``Expressive power of tensor-network factorizations for probabilistic modeling,'' in \textit{Adv. Neural Inf. Process. Syst.}, vol. 32, 2019.

\bibitem{cheng2019tree}
S. Cheng, L. Wang, T. Xiang, and P. Zhang, ``Tree tensor networks for generative modeling,'' \textit{Phys. Rev. B}, vol. 99, no. 15, p. 155131, 2019.

\bibitem{shi2006classical}
Y.-Y. Shi, L.-M. Duan, and G. Vidal, ``Classical simulation of quantum many-body systems with a tree tensor network,'' \textit{Phys. Rev. A}, vol. 74, no. 2, p. 022320, 2006.

\bibitem{vieijra2022generative}
T. Vieijra, L. Vanderstraeten, and F. Verstraete, ``Generative modeling with projected entangled-pair states,'' \textit{arXiv:2202.08177}, 2022.

\bibitem{liu2019machine}
D. Liu \textit{et al.}, ``Machine learning by unitary tensor network of hierarchical tree structure,'' \textit{New J. Phys.}, vol. 21, no. 7, p. 073059, 2019.

\bibitem{nie2025deep}
C. Nie, J. Chen, and Y. Chen, ``Deep tree tensor networks for image recognition,'' \textit{arXiv:2502.09928}, 2025.

\bibitem{zhou2024tensor}
W. Zhou \textit{et al.}, ``Tensor star tensor decomposition and its applications to higher-order compression and completion,'' \textit{arXiv:2403.10481}, 2024.

\bibitem{uschmajew2020geometric}
A. Uschmajew and B. Vandereycken, ``Geometric methods on low-rank matrix and tensor manifolds,'' in \textit{Handbook of Variational Methods for Nonlinear Geometric Data}. Springer, 2020, pp. 261--313.

\bibitem{yang2025space}
Y. Yang, B. Gao, and Y.-X. Yuan, ``A space-decoupling framework for optimization on bounded-rank matrices with orthogonally invariant constraints,'' \textit{Math. Program.}, 2026, doi: 10.1007/s10107-026-02331-7.

\bibitem{absil2008optimization}
P.-A. Absil, R. Mahony, and R. Sepulchre, \textit{Optimization Algorithms on Matrix Manifolds}. Princeton, NJ, USA: Princeton Univ. Press, 2008.

\bibitem{klemetsen2022neural}
C. B. Klemetsen, ``Neural networks on low-rank and Stiefel manifolds,'' M.S. thesis, NTNU, 2022.

\bibitem{peng2025normalized}
R. Peng, C. Zhu, B. Gao, X. Wang, and Y.-X. Yuan, ``Normalized tensor train decomposition,'' \textit{arXiv:2511.04369}, 2025.

\bibitem{saragadam2024deeptensor}
V. Saragadam \textit{et al.}, ``DeepTensor: Low-rank tensor decomposition with deep network priors,'' \textit{IEEE Trans. Pattern Anal. Mach. Intell.}, vol. 46, no. 12, pp. 10337-10348, 2024.


\bibitem{jacob2020structured}
M. Jacob, M. P. Mani, and J. C. Ye, ``Structured low-rank algorithms: Theory, magnetic resonance applications, and links to machine learning,'' \textit{IEEE Signal Process. Mag.}, vol. 37, no. 1, pp. 54--68, 2020.

\bibitem{white1992density}
S. R. White, ``Density matrix formulation for quantum renormalization groups,'' \textit{Phys. Rev. Lett.}, vol. 69, no. 19, p. 2863, 1992.

\bibitem{liu2022tensor}
Y. Liu, J. Liu, Z. Long, and C. Zhu, \textit{Tensor Computation for Data Analysis}. Springer, 2022.

\bibitem{LM08}
A.S. Lewis and J. Malick. ``Alternating projections on manifolds," \textit{Math. Oper. Res.}, vol. 33, no. 1, pp. 216-234, 2008.

\bibitem{stoudenmire2012studying}
E. M. Stoudenmire and S. R. White, ``Studying two-dimensional systems with the density matrix renormalization group,'' in \textit{Annu. Rev. Condens. Matter Phys.}, vol. 3, no. 1, 2012, pp. 111--128.

\bibitem{fischer2012introduction}
A. Fischer and C. Igel, ``An introduction to restricted Boltzmann machines,'' in \textit{Proc. Iberoamerican Conf. Pattern Recognit.}, 2012, pp. 14--36.

\bibitem{orus2014practical}
R. Orús, ``A practical introduction to tensor networks: Matrix product states and projected entangled pair states,'' \textit{Ann. Phys.}, vol. 349, pp. 117--158, 2014.


\bibitem{sakai2024general}
H. Sakai and H. Iiduka, ``A general framework of Riemannian adaptive optimization methods with a convergence analysis,'' \textit{arXiv:2409.00859}, 2024.

\bibitem{nar2018step}
K. Nar and S. Sastry, ``Step size matters in deep learning,'' in \textit{Adv. Neural Inf. Process. Syst.}, vol. 31, 2018.

\bibitem{becigneul2018riemannian}
G. Bécigneul and O.-E. Ganea, ``Riemannian adaptive optimization methods,'' \textit{arXiv:1810.00760}, 2018.

\bibitem{mahoney2011randomized}
M. W. Mahoney, ``Randomized algorithms for matrices and data,'' \textit{Found. Trends Mach. Learn.}, vol. 3, no. 2, pp. 123--224, 2011.


\bibitem{zhang2016riemannian}
H. Zhang, S. J. Reddi, and S. Sra, ``Riemannian SVRG: Fast stochastic optimization on Riemannian manifolds,'' in \textit{Adv. Neural Inf. Process. Syst.}, vol. 29, 2016.

\bibitem{vidal2003efficient}
G. Vidal, ``Efficient classical simulation of slightly entangled quantum computations,'' \textit{Phys. Rev. Lett.}, vol. 91, no. 14, p. 147902, 2003.

\bibitem{verstraete2006criticality}
F. Verstraete, M. M. Wolf, D. Perez-Garcia, and J. I. Cirac, ``Criticality, the area law, and the computational power of projected entangled pair states,'' \textit{Phys. Rev. Lett.}, vol. 96, no. 22, p. 220601, 2006.


\bibitem{wu2025algorithms}
Y. Wu and Z. Dai, ``Algorithms for variational Monte Carlo calculations of fermion PEPS in the swap gates formulation,'' \textit{arXiv:2506.20106}, 2025.





\bibitem{tang2025gauging}
W. Tang, L. Vanderstraeten, and J. Haegeman, ``Gauging the variational optimization of projected entangled-pair states,'' \textit{arXiv:2508.10822}, 2025.

\end{thebibliography}
\end{document}